\documentclass{article}

\usepackage[main, final, nonatbib]{neurips_2026}

\usepackage[utf8]{inputenc} 
\usepackage[T1]{fontenc}    
\usepackage{hyperref}       
\usepackage{url}            
\usepackage{booktabs}       
\usepackage{amsfonts}       
\usepackage{nicefrac}       
\usepackage{microtype}      
\usepackage{xcolor}         

\usepackage{amssymb}
\usepackage{amsthm}
\usepackage{graphicx}
\usepackage{dsfont}
\usepackage{amsmath}
\usepackage{algorithm}
\usepackage{algpseudocode}
\usepackage{setspace}
\usepackage{multirow}
\usepackage{wrapfig}
\usepackage{xspace}
\usepackage{etoolbox}

\DeclareMathOperator*{\argmin}{arg\,min}

\newcommand\indep{\protect\mathpalette{\protect\independenT}{\perp}}
\def\independenT#1#2{\mathrel{\rlap{$#1#2$}\mkern2mu{#1#2}}}
\def\prob{\mathbb{P}} 
\def\P{\prob_{\mathrm{obs}}}
\def\E{\mathbb{E}} 
\def\L{\mathcal{L}}
\def\R{\mathbb{R}}
\newcommand{\ind}{\mathds{I}} 

\newcommand{\KL}{\mathrm{KL}}

\newcommand{\Var}{\mathrm{Var}}

\usepackage[textsize=tiny]{todonotes} 

\usepackage{pifont}
\newcommand{\cmark}{\textcolor{green!70!black}{\ding{51}}}  
\newcommand{\xmark}{\textcolor{red!80!black}{\ding{55}}}    

\newtheorem{theorem}{Theorem}

\newtheorem{remark}{Remark}

\usepackage[square,numbers]{natbib}
\definecolor{darkgreen}{RGB}{0,100,0}

\usepackage{subcaption}
\usepackage{mathtools}
\usepackage[capitalize,noabbrev]{cleveref}
\usepackage{cancel}
\usepackage{placeins}

\AtBeginEnvironment{algorithmic}{\footnotesize}
\AtBeginEnvironment{algorithm}{\footnotesize}

\usepackage{hyperref}
\usepackage[capitalize,noabbrev]{cleveref}

\usepackage{capt-of}

\title{Generalized Bayes for Causal Inference}

\author{
  \textbf{Emil Javurek}\thanks{Correspondence to: \texttt{emil.javurek@lmu.de}},\, \textbf{Dennis Frauen}, \textbf{Yuxin Wang}, \textbf{Stefan Feuerriegel}\\
  LMU Munich \\
  Munich Center of Machine Learning (MCML)
}

\begin{document}

\maketitle

\begin{abstract}
Uncertainty quantification is central to many applications of causal machine learning, yet principled Bayesian inference for causal effects remains challenging. Standard Bayesian approaches typically require specifying a probabilistic model for the data-generating process, including high-dimensional nuisance components such as propensity scores and outcome regressions. Standard posteriors are thus vulnerable to strong modeling choices, including complex prior elicitation. In this paper, we propose a \textit{generalized Bayesian framework for causal inference}. Our framework avoids explicit likelihood modeling; instead, we place priors directly on the causal estimands and update these using an identification-driven loss function, which yields \textit{generalized posteriors} for causal effects. As a result, our framework turns existing loss-based causal estimators into estimators with full uncertainty quantification. Our framework is flexible and applicable to a broad range of causal estimands (e.g., ATE, CATE). Further, our framework can be applied on top of state-of-the-art causal machine learning pipelines (e.g., Neyman-orthogonal meta-learners). For Neyman-orthogonal losses, we show that the generalized posteriors converge to their oracle counterparts and remain robust to first-stage nuisance estimation error. With calibration, we thus obtain valid frequentist uncertainty even when nuisance estimators converge at slower-than-parametric rates. Empirically, we demonstrate that our proposed framework offers causal effect estimation with calibrated uncertainty across several causal inference settings. To the best of our knowledge, this is the first flexible framework for constructing generalized Bayesian posteriors for causal machine learning.

\end{abstract}



\section{Introduction}
\vspace{-0.3cm}


Causal inference is widely used to assess treatment effects in order to compare interventions and inform decisions \cite{Feuerriegel.2024, banerjiClinicalAITools2023, heckmanMakingMostOut1997}. However, conclusions are typically not based on the point estimates of causal effects alone, but rather on considering uncertainty associated with estimating causal effects. For example, in medicine, a treatment that appears beneficial on average may still be clinically unacceptable if the uncertainty is large enough to imply a substantial probability of harm \cite{zampieriUsingBayesianMethods2021, kneibRageMeanReview2023}. As a result, many applications of causal inference, require not only estimates of causal effects but also rigorous \textit{uncertainty quantification}.


Bayesian inference provides a natural framework for uncertainty quantification by computing a posterior distribution over the quantity of interest. However, in causal inference, standard Bayesian approaches are subject to fundamental \textit{challenges}: existing approaches typically require specifying a probabilistic model of the observational data-generating process, including nuisance components such as propensity scores and outcome regressions (see Section~\ref{sec:related_work} for a detailed overview). This is cumbersome and the resulting posterior is fragile to the modeling choices. Moreover, a prior must be placed on the nuisance components, which can interact with the likelihood in unintended ways, such as introducing regularization-induced confounding \cite{hahnRegularizationConfoundingLinear2016}.


\begin{figure*}
\centering
\includegraphics[width=0.9\textwidth]{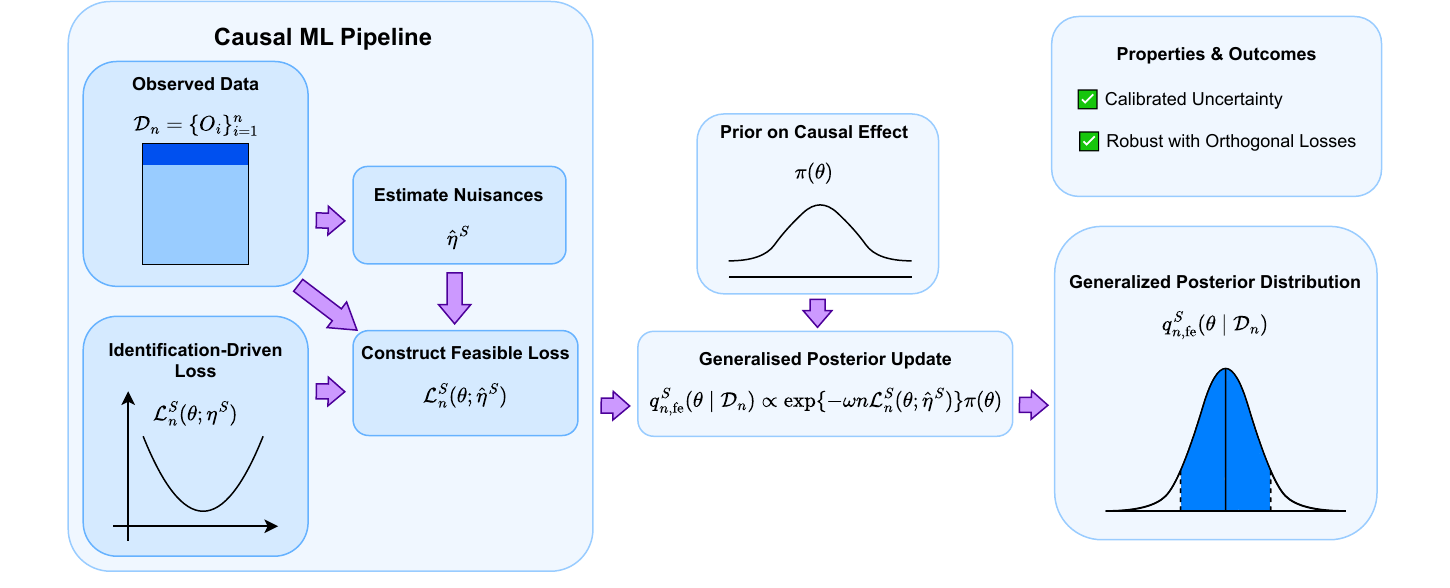}
\captionsetup{width=\textwidth} 
\caption{
\scriptsize{
\textbf{Overview of the pipeline for generalized Bayesian inference:} In the standard causal ML pipeline, one observes the data $\mathcal{D}_n$, chooses an identification-driven loss, estimates the corresponding nuisances, and thus constructs a fitting objective. Here, we further elicit a prior on the causal effect and together construct a generalized updating rule yielding the generalized posterior distribution for the causal effect of interest. This approach can be calibrated to yield frequentist-valid uncertainty and be made robust to nuisance estimation error via Neyman-orthogonal losses. 
}
}
\label{fig:overview}
\vspace{-0.5cm}
\end{figure*}

In this paper, we propose a \textit{\textbf{generalized Bayesian framework for causal inference}} that avoids explicit likelihood modeling.\footnote{By ``generalized'', we mean Bayesian updating based on a loss function rather than a likelihood, following the generalized (specifically, Gibbs) posterior framework of \citet{Bissiri.2016}. Standard Bayesian inference is recovered as a special case when the loss is the negative log-likelihood.} Conceptually, we follow recent works on generalized Bayes \citep{JeremiasKnoblauch.2022} and view Bayesian inference \textit{not} just through the Bayes rule formula, but frame it as a principled belief-updating procedure that naturally extends to loss-based formulations more compatible with modern machine learning (ML). In this framework, we thus avoid specifying a probabilistic model of the full observational data-generating process; instead, we (i)~place priors directly on causal estimands and (ii)~update these using identification-driven loss functions for which the minimizers recover the targeted causal effect. This construction yields generalized posteriors for causal effects, without requiring priors or likelihoods for high-dimensional nuisance components. This addresses key challenges of likelihood-based Bayesian approaches, which can be sensitive to nuisance component modeling misspecification. Consequently, our approach brings the following benefits:

$\bullet$\,\textbf{Flexibility.} Our framework is applicable to a wide range of causal estimands. Throughout our paper, we illustrate our framework using two canonical estimands from the causal inference literature: (i)~the average treatment effect (ATE) and (ii)~the conditional average treatment effect (CATE) settings. The proposed framework extends naturally to any other identified causal estimand that can be estimated with a loss. Moreover, our framework can be applied on top of state-of-the-art causal ML pipelines, in particular Neyman-orthogonal meta-learners such as the DR-learner \citep{Kennedy.30042020, vanderLaan.2011}. Altogether, our framework provides a general recipe for transforming existing loss-based causal estimators into estimators with full uncertainty quantification.

$\bullet$\,\textbf{Theoretical guarantees.} For Neyman-orthogonal losses, we show that the generalized posteriors have several favorable properties. In particular, the generalized posterior converges to the oracle counterpart and remains robust to first-stage nuisance estimation error, even when nuisance estimators converge at slower-than-parametric rates (analogous to state-of-the-art Neyman-orthogonal meta-learners for point estimates). It follows that the estimated generalized posterior also satisfies a Bernstein-von Mises limit, yielding asymptotically valid frequentist uncertainty quantification. This establishes how orthogonal statistical learning applies to generalized Bayesian causal inference.


$\bullet$\,\textbf{Practical advantages.} \textit{First}, our framework enables the placement of priors directly on the causal estimand of interest, which allows prior information to be specified transparently and avoids the need to encode prior beliefs indirectly through high-dimensional nuisance models, where their effect on the causal estimand is difficult to interpret or control. \textit{Second}, this loss-based formulation makes it possible to characterize and control how nuisance estimation error propagates into posterior uncertainty through properties of the loss such as Neyman-orthogonality. \textit{Third}, the resulting generalized posteriors can be calibrated to match frequentist uncertainty and thus achieve coverage.


Our main \textbf{contributions} are:\footnote{\raggedright Code is available at \url{https://github.com/EmilJavurek/Generalised-Bayes-for-Causal-Inference}}
\textbf{(1)}~We introduce a generalized Bayesian framework for causal inference, which is applicable to state-of-the-art causal ML pipelines.  \textbf{(2)}~We prove that, for Neyman-orthogonal losses, the resulting generalized posterior is robust to first-stage nuisance estimation error and thus offers explicit convergence guarantees for the posterior. \textbf{(3)}~We empirically demonstrate that our framework yields frequentist-calibrated uncertainty quantification across several causal inference settings. Importantly, the goal of our framework is \textit{not} to approximate a standard likelihood-based Bayesian posterior, but to do loss-based updates of prior beliefs that yield \textit{generalized} posteriors with valid frequentist uncertainty. To the best of our knowledge, ours is the first framework that constructs generalized posteriors for causal inference using identification-driven losses with formal robustness guarantees wrt. the nuisance estimation error.

\vspace{-0.4cm}
\section{Related Work}
\label{sec:related_work}
\vspace{-0.2cm}

\textbf{Causal ML:} Recent work combines causal inference with machine learning to develop flexible methods for estimating a wide range of causal estimands, including the ATE \cite{Chernozhukov.2018,Kennedy.12032022, DanielRubin.2006} and CATE \cite{Foster.1252019, Nie.13122020, Kennedy.30042020}. 
State-of-the-art methods are typically \textit{Neyman-orthogonal meta-learners} (e.g., the DR-learner \citep{vanderlaanStatisticalInferenceVariable2006,Kennedy.30042020}, R-learner \citep{Nie.13122020}), which use flexible machine learning models to estimate nuisance components and then target the causal effect in the second stage with a Neyman-orthogonal loss that ensures first-order insensitivity to nuisance estimation error \citep{Foster.1252019, Chernozhukov.2018, Kennedy.30042020}.

\textbf{Generalized Bayesian inference:} Generalized Bayesian inference constructs posterior distributions by updating a prior using a loss function, rather than a likelihood. Conceptually, we may understand this generalization through the role the likelihood/losses play in the belief updating procedure: \textit{both simply quantify the discrepancy between a candidate parameter value and the observed data}. \citet{Bissiri.2016} showed that this loss-based update yields a principled posterior under a decision-theoretic formulation, where standard Bayesian inference is recovered as a special case when the loss is the negative log-likelihood
. Further works have extended the theory of generalized Bayes \citep{matsubaraRobustGeneralisedBayesian2022, Knoblauch.12102019, Knoblauch., Gao.5242023}; \textit{however, to the best of our knowledge, no existing work has explored the extensions of generalized Bayes theory to causal inference - specifically, the question of employing losses depending on estimated nuisances.}


\textbf{Likelihood-based Bayesian causal inference: }

A large body of Bayesian causal inference proceeds by specifying a probabilistic model for (all parts of) the observed data distribution and performing posterior inference under this model. Many approaches leverage flexible Bayesian nonparametric priors over high-dimensional nuisance components \citep[e.g.,][]{Antonelli.5132018, Ray.2020, Munoz.2011}; see \citet{Linero.1162021} for an overview. \textit{However, this entails modeling high-dimensional nuisance components and specifying corresponding priors, introducing challenges for robustness and prior elicitation (which we discuss further in Section~\ref{sec:standard-bayes-hard}).}


In parallel, Bayesian approaches have been developed for specific settings but tied to specific model classes \citep[e.g.,][]{zhangTreatmentEffectEstimation2021,Wang.6162021,Goh.6132020, chernozhukovMCMCApproachClassical2003,zhangTreatmentEffectEstimation2021,alaaBayesianInferenceIndividualized2017}. These methods typically embed causal structure into a chosen probabilistic model (e.g., multi-task Gaussian process models for CATE estimation, which impose Gaussian process priors;  \citet{alaaBayesianInferenceIndividualized2017}). \textit{However, such approaches are inherently model-specific and their usefulness hinges on the model class (the inductive bias) matching the problem at hand.}



A key challenge for Bayesian causal inference is achieving robustness to nuisance estimation error. In joint Bayesian models, posterior feedback between outcome and propensity-score components can undermine robustness under misspecification, motivating modular and cut-posterior approaches \citep{stephensCausalInferenceMisspecification2022,Saarela.2016,jacobBetterTogetherStatistical2017a}, including their extension to generalized Bayesian inference \citep{frazierCuttingFeedbackModularized2025}. Other approaches pursue reliable uncertainty quantification for causal functionals through tailored doubly robust constructions or semiparametric posterior corrections \citep{Breunig.11292022,Orihara.652025,sertBayesianSemiparametricCausal2025,Babasaki.9102024,yiuSemiparametricPosteriorCorrections2025}. \textit{However, these approaches do not provide a general construction applicable across varying causal settings or estimands.} 

We elaborate on the difficulties of standard likelihood-based Bayesian approaches in Section~\ref{sec:standard-bayes-hard}. Extended related work can be found in Appendix~\ref{appx:extended-related-work}.

\textbf{Research gap:} Overall, existing Bayesian causal methods largely rely on model- or estimand-specific constructions and often lack robustness to nuisance misspecification, which limits their applicability as a general framework for uncertainty quantification. A unified framework across estimands and models for (generalized) Bayesian causal inference is still missing.

\vspace{-0.3cm}
\section{Problem Setup}
\vspace{-0.1cm}

\textbf{Notation:} We denote random variables by capital letters $X,A,Y,Z,O$ and their realizations by small letters $x,a,y,z,o$ from domains $\mathcal{X},\mathcal{A},\mathcal{Y},\mathcal{Z},\mathcal{O}$. Let $\mathbb{P}(O)$ denote a distribution of some random variable $O$, and let $p(O=o)$ be a corresponding density or probability mass function. Let $\mathcal{P}(\mathcal{O})$ denote the set of all probability distributions over $\mathcal{O}$. We write $\E_{\prob}$ to denote expectations with respect to the distribution $\prob$ and $\E_n$ to denote the corresponding empirical mean over an observed sample of size $n$ from $\prob$. We use the potential outcomes framework \citep{Rubin.1974}.

\textbf{Setting:} We consider a \textit{causal inference setting} given by a tuple $\mathcal{C} = \bigl(O, \mathcal{P}_\mathrm{obs} \times \mathcal{P}_\mathrm{int}, Q \bigr)$, where $O$ collects observed variables containing at least the treatment $A$ and the outcome $Y$;
$(\prob_\mathrm{obs},\prob_\mathrm{int})\in\mathcal{P}_\mathrm{obs}\times\mathcal{P}_\mathrm{int}$ are paired observational/interventional distributions over $O$ corresponding to an intervention on $A$; and $Q(\prob_\mathrm{int}) \in \Theta$ is the causal estimand of interest.

The \textbf{causal estimand} $Q(\prob_\mathrm{int})$ is said to be \textit{identifiable} if, under identification assumptions specific to the setting $\mathcal{C}$, there exists a measurable functional $\Bar{Q}$ such that $Q(\prob_\mathrm{int})$ = $\Bar{Q}(\prob_\mathrm{obs}) \equiv \theta$. Given an observed dataset $\mathcal{D}_n = \{O_i\}_{i=1}^n \sim \prob_\mathrm{obs}$, we estimate the identified causal estimand via $\theta \approx \hat{\theta}(\mathcal{D}_n)$.

\textbf{Estimation:} Given an observed dataset $\mathcal{D}_n = \{O_i\}_{i=1}^n \sim \prob_\mathrm{obs}$ from setting $\mathcal{C}$, we aim to estimate the identified causal estimand $\theta$ via an estimator $\hat{\theta}(\mathcal{D}_n)$. For a given estimand $\theta$, there are typically multiple valid estimation strategies; here, we focus on a broad class of M-estimation methods. Formally, a strategy $S \in \mathcal{S}$ obtains a point-estimator $\hat{\theta}^{S}$ by a loss minimization problem
\begin{align}
    \hat{\theta}^{S} \in 
    \argmin_{\theta \in \Theta} \L^{S}_n (\theta)
    \equiv \argmin_{\theta \in \Theta} \E_n\bigl[\ell^{S}(O;\theta) \bigr],
\end{align}
\vspace{-0.2cm}

where $\L^{S}_{n}(\theta)$ is the empirical loss (using $\mathcal{D}_n$), and $\ell^{S}(O; \theta)$ is the per-observation loss. 

\textbf{Nuisances:} In causal inference, estimation strategies $S$ often depend on additional \emph{nuisances} $\eta^S=\eta^S(\prob_\mathrm{obs})$. These nuisances (e.g., propensity scores, outcome regressions or conditional densities) are not themselves the target causal estimand, but are used as part of the estimation procedure. Accordingly, we define an estimation strategy $S$ to include both the nuisances $\eta^S$ and an estimator $\hat\eta^S=\hat\eta^S(\mathcal{D}_n)$ for the nuisances. Correspondingly, we write $\ell^{S}(O;\theta,\hat\eta^{S})$ for the nuisance dependent per-observation loss. At the population level, losses are typically motivated by an \emph{oracle} score or pseudo-outcome constructed using the true nuisance $\eta^S$. In practice, replacing $\eta^S$ by $\hat\eta^S$ is unavoidable and induces \emph{nuisance estimation error}, which can dominate the error of $\hat\theta^S$ in complex/high-dimensional settings. Our framework is agnostic to how nuisance estimators are obtained, nevertheless, we discuss options in Appendix~\ref{appx:nuisance-estimation}. A central goal of our later theory will be to characterize how such nuisance error propagates downstream into the generalized posterior.

Our framework applies to a broad class of causal estimands that fall under the above causal inference setting; to make the exposition concrete, we introduce a set of running examples below, which we use throughout the paper to illustrate the construction, derivations, and theoretical results.

$\blacksquare$\,\textbf{Example 1 (back-door ATE).}
Consider $O=(X,A,Y)\sim\prob_\mathrm{obs}$ where $X$ are covariates, $A\in\{0,1\}$ is a binary treatment, and $Y$ is an outcome. Let $Y(a)$ denote the potential outcome under treatment $A=a$ and write $\prob_\mathrm{int}$ for the interventional distribution induced by setting $A \leftarrow a$. The \textit{average treatment effect} is given by $\theta_{\mathrm{ATE}} \;=\; \E\bigl[Y(1)-Y(0)\bigr]$. 
We impose standard identification assumptions \citep{Robins.2000}: 
(i) \emph{consistency} $Y=Y(A)$,
(ii) \emph{positivity} $\prob_\mathrm{obs}(A=a\mid X=x)>0$ a.s. for $a\in\{0,1\}$, and
(iii) \emph{unconfoundedness} $(Y(0),Y(1))\indep A\mid X$.
Under these conditions, the target is identified as functional of $\prob_\mathrm{obs}$. Writing $m_a(x)=\E_{\prob_\mathrm{obs}}[Y\mid A=a,X=x]$, we have $\theta_{\mathrm{ATE}} = \E_{\prob_\mathrm{obs}}\bigl[m_1(X)-m_0(X)\bigr]$. Scalar estimands such as the ATE often admit closed-form minimizers, such as $\hat{\theta}^{S} = \tfrac{1}{n}\sum_{i=1}^{n} \hat{Y}^{S}(O_i)$ with strategy $S \in \{\mathrm{RA},\mathrm{IPW},\mathrm{AIPW}\footnote{For the ATE, the pseudo-outcomes $\hat Y^{\mathrm{AIPW}}$ from the $\mathrm{AIPW}$ strategy directly correspond to the $\hat Y^{\mathrm{DR}}$ from CATE estimation.}\}$ determining the pseudo-outcome form $\hat{Y}^{S}$. With finite-dimensional problems where the estimator admits a closed-form expression, the estimation strategy can nevertheless be represented as the minimizer of an appropriate loss, provided the loss is from a Fisher-consistent family of losses (e.g., squared loss for mean-type estimands, absolute losses for medians, or pinball losses for quantiles) \citep{gneitingStrictlyProperScoring2007}. Since the ATE is a mean, we may thus define a loss $\ell^{S}(O_i;\theta,\hat \eta^{S}) = (\hat Y^{S}(O_i;\hat \eta^{S}) - \theta)^2$ for any of the above strategies $S$. 

$\blacksquare$\,\textbf{Example 2 (back-door CATE).}
In the same setting as Example~1, we consider the \textit{conditional average treatment effect} (CATE), defined as $\theta_{\mathrm{CATE}}(x) = \E\bigl[Y(1)-Y(0)\mid X=x\bigr]$. Under the same  assumptions, the CATE is identified by the observational conditional mean functions as $\theta_{\mathrm{CATE}}(x) = m_1(x)-m_0(x)$. A common approach for CATE estimation is to construct a pseudo-outcome $\hat{Y}^{S}(O;\eta^{S})$ satisfying $\E_{\prob_\mathrm{obs}}[\hat{Y}^{S}(O;\eta^{S}) \mid X = x] = \theta_\mathrm{CATE}(x)$ and fit $\hat\theta(\cdot)$ by regression using the squared error loss $\ell^{S}(O;\theta,\hat\eta^{S}) = \bigl(\hat{Y}^{S}(O;\hat\eta^{S}) - \theta(X)\bigr)^2$.
Typical choices for strategy include regression adjustment (RA),
{\small $\hat{Y}^{\mathrm{RA}} = \hat m_1(X)-\hat m_0(X)$};
inverse propensity weighting (IPW),
{\small $\hat{Y}^{\mathrm{IPW}} =
\tfrac{AY}{\hat e(X)} -
\tfrac{(1-A)Y}{1-\hat e(X)}$};
and doubly robust estimation (DR),
{\small $\hat{Y}^{\mathrm{DR}} =
\hat{Y}^{\mathrm{RA}} +
\bigl(\tfrac{A}{\hat e(X)} -
\tfrac{1-A}{1-\hat e(X)}\bigr)
(Y-\hat m_A(X))$}.



\vspace{-0.3cm}
\section{Why standard Bayesian approaches for causal inference are difficult}
\label{sec:standard-bayes-hard}
\vspace{-0.2cm}

The na{\"i}ve approach to Bayesian inference of a causal query $\theta$ (in a fixed causal setting $\mathcal{C}$) would be to simply follow the classical Bayes rule $q(\theta | \mathcal{D}_n) \propto \prob_\mathrm{obs}(\mathcal{D}_n|\theta)\pi(\theta)$\footnote{By $\propto$ we denote proportionality up to a multiplicative constant.}. 
While specifying a prior $\pi(\theta)$ directly on $\theta \in \Theta$ does not pose any issues -- and is in fact desirable for encoding belief into prior -- the likelihood $\prob_\mathrm{obs}(\mathcal{D}_n|\theta)$ is problematic. In causal settings, $\theta$ typically does not fully determine the observed data distribution $\prob_\mathrm{obs}$, so $\prob_\mathrm{obs}(\mathcal{D}_n|\theta)$ is not directly defined without further modeling of the data-generating process. For example:

$\blacksquare$\,\textbf{Example 1 (back-door ATE).}
Many different distributions $\P$ of $(X,A,Y)$ induce the \textit{same} ATE $\theta_{\mathrm{ATE}}$. As a result, $\theta_{\mathrm{ATE}}$ does \textit{not} uniquely determine a value of the likelihood for $(X,A,Y)$, i.e. $\P(\mathcal{D}_n \mid \theta_{\mathrm{ATE}})$ is \textit{not} well-defined.

Instead, a proper likelihood is defined as $p(\mathcal{D}_n \mid \prob_\mathrm{obs})$, with a corresponding prior $\pi(\prob_\mathrm{obs})$ over the space of data-generating distributions. This induces the posterior distribution 
\begin{align}
    q(\theta \mid \mathcal{D}_n) = \int \delta_{\theta}(\bar Q(\prob_\mathrm{obs})) p(\prob_\mathrm{obs} \mid \mathcal{D}_n) \;\rm d \prob_\mathrm{obs} ,
\end{align} 
which is a marginalization of the joint posterior $p(\prob_\mathrm{obs} | \mathcal{D}_n) \propto \pi(\prob_\mathrm{obs}) p(\mathcal{D}|\prob_\mathrm{obs})$, and where $\delta_{\theta}(\bar Q(\prob_\mathrm{obs}))$ denotes a Dirac delta of $\bar Q(\prob_\mathrm{obs})$ at $\theta$. 

To make the prior distribution and likelihood at the level of $\prob_\mathrm{obs}$ tractable, we restrict the space of possible data-generating distributions to a (semi-) parameterized subset $\mathcal{Q}_\mathrm{obs} \subset \mathcal{P}_\mathrm{obs}$. For example:

$\blacksquare$\,\textbf{Examples 1 \& 2 (back-door ATE \& CATE):}
In the back-door adjustment settings, we may parameterize a model of the data-generating process $\prob_\mathrm{obs}$ by  $\prob_{\mathrm{obs},\xi}(O) = p_{\xi_{x}}(X)p_{\xi_{a}}(A \mid X)p_{\xi_{y}}(Y \mid A,X)$ from a corresponding model class $Q_{\mathrm{obs},\xi} = \bigl\{\prob_{\mathrm{obs},\xi}(O)  \mid  \xi = (\xi_{x},\xi_{a},\xi_{y}) \in \Xi \bigr\}$, together with some prior $\pi(\xi) \in \mathcal{P}(\Xi)$ on the parameters $\xi$. Notice that $\xi_{a}$ and $\xi_{y}$ actually parametrize the propensity and outcome regression nuisances. The standard Bayes approach thus forces us to specify a prior over nuisances and to compute a posterior for the entire $\xi$, which will typically be higher-dimensional than $\theta$, possibly even infinite-dimensional. The posterior for $\theta$ must then finally be computed as a marginalization of the overall posterior over all (model) parameters $\xi$ of the data-generating process.

As illustrated by the above example, several practical difficulties arise even in seemingly simple causal settings. In particular, the following issues arise:

\textbf{(1) Model misspecification:} The restricted model class $\mathcal{Q}_\mathrm{obs} \subset \mathcal{P}_\mathrm{obs}$ may fail to contain the true observational distribution $\prob_\mathrm{obs}^\star$, especially if the model class is overly restrictive or otherwise misspecified. In this case, the likelihood is misspecified and the correct Bayesian posterior is \textit{not} obtained. Nevertheless, such restrictions are often unavoidable in practice for computational and regularization reasons.

\textbf{(2) Priors on high-dimensional nuisance functions.} As illustrated in the back-door example, specifying a likelihood requires placing priors on nuisance components such as propensity scores and outcome regressions. Even after restricting $\mathcal{Q}_\mathrm{obs}$, these nuisance objects are typically high-dimensional functions, for which posterior convergence and Bernstein–von Mises approximation are \textit{not} guaranteed \textit{unless strong additional assumptions are enforced } \citep{freedmanBernsteinVonMisesTheorem1999}.

\textbf{(3) Sensitivity to nuisance priors, including feedback problems.} The (marginalized) posterior over the causal estimand $\theta$ depends on an appropriate prior specification for the nuisance components but correct prior elicitation is difficult. For example, even when the propensity and outcome nuisances are given independent priors, a joint likelihood couples their posteriors; outcome information can then feed back into propensity estimation, which may harm balance and robustness under misspecification (the feedback problem) \cite{stephensCausalInferenceMisspecification2022, Orihara.652025}.

\textbf{(4) Lack of direct control over the induced prior on the causal estimand.}
Placing a prior on $\prob_\mathrm{obs}$ induces also a prior on the causal estimand $\theta$ via a non-trivial mapping $\pi(\theta) = \int \delta_{\bar Q(\prob_\mathrm{obs})}(\theta)  
\pi(\prob_\mathrm{obs}) \; \rm d \prob_\mathrm{obs}$. As a consequence, it is difficult -- if not impossible -- to directly elicit prior beliefs about $\theta$. This is critical if we want to follow the philosophy of Bayesian updating of beliefs; following the standard Bayes rule formula in causal inference settings may thus make the task of eliciting a prior belief on $\theta$ difficult.

With the understanding of how standard Bayesian inference in causal inference is made difficult by nuisance functions, we now turn to generalized Bayes for an alternative solution.

\vspace{-0.3cm}
\section{Method (Generalized Bayes)}
\vspace{-0.3cm}

We proceed in four steps. In §\ref{sec:Gibbs_posteriors}, we introduce generalized posteriors based on identification-driven losses, which define a principled Bayesian update of belief for causal estimands. In §\ref{sec:feasible_posteriors}, we turn to the realistic case where nuisances must be estimated 
and, in §\ref{sec:algorithm}, we present an algorithm for computing the corresponding feasible generalized posteriors. In §\ref{sec:theoretical_results}, we provide our main theoretical result: for Neyman-orthogonal losses with cross-fitting, the feasible posterior converges to its oracle counterpart and yields valid uncertainty quantification.

\vspace{-0.2cm}
\subsection{Generalized posteriors}
\label{sec:Gibbs_posteriors}
\vspace{-0.2cm}

To define a generalized posterior, we consider a fixed causal setting $\mathcal{C}$ with an identified estimand $\theta^\star \in \Theta$ and an estimation strategy $S$ with loss $\L^{S}_{n}(\theta ;\eta^{S})$. Let $\pi(\theta) \in \mathcal{P}(\Theta)$ denote a prior distribution on the causal estimand, and let $\omega > 0$ be a calibration parameter. Then, the corresponding \textit{generalized 
posterior} \citep{Bissiri.2016} is defined\footnote{This version of generalized posteriors is also sometimes referred to as Gibbs posteriors. For other generalization of the Bayes posterior, see, for example, \citep{JeremiasKnoblauch.2022}.} as
\begin{align}
    q^{S}_{n}(\theta \mid \mathcal{D}_n) \propto \exp\bigl\{-\omega n \L^{S}_{n}(\theta ;\eta^{S}) \bigr\} \pi(\theta) .
\end{align}
This construction turns any loss-based estimator of a causal estimand $\theta^\star$ into a posterior distribution over $\Theta$ \textit{without} specifying a likelihood or modeling the observational distribution $\prob_\mathrm{obs}$ (observe that $\mathcal{L}$ set to $\tfrac{1}{\omega}$-scaled negative log-likelihood recovers standard Bayes posterior; for detailed interpretation see Appendix~\ref{appx:interpretation}). Equivalently, the same generalized posterior can be characterized as the solution to an optimization problem over distributions

\vspace{-0.3cm}
\begin{align*}
    q^{S}_{n}(\theta \mid \mathcal{D}_n) = \argmin_{q \in \mathcal{P}(\Theta)}
    \left\{ 
    \E_{\theta \sim q}\bigl[\omega n \L^{S}_{n}(\theta ;\eta^{S}) \bigr] + \KL( q \parallel \pi)
    \right\}
\end{align*}
\vspace{-0.3cm}

where $\KL(\cdot\parallel\cdot)$ denotes the Kullback--Leibler divergence \citep{JeremiasKnoblauch.2022}. This offers a variational formulation, which is particularly useful in practice, as simply replacing $\mathcal{P}(\Theta)$ by a desired variational family yields a tractable variational objective for approximating the generalized posterior.

\vspace{-0.2cm}
\subsection{Feasible posteriors}
\label{sec:feasible_posteriors}
\vspace{-0.1cm}

The generalized posterior defined above assumes that the nuisance components entering the loss are known. In practice, however, these nuisances must be estimated from data. 
Let $\P$ denote the true observational distribution, and write
$\eta^S_0 := \eta^S(\P)$ for the corresponding true (oracle) nuisances of the estimation strategy $S$. We define the \emph{oracle} generalized posterior as \mbox{$q^{S}_{n,\mathrm{or}}(\theta \mid \mathcal{D}_n)
:= q^{S}_{n}(\theta \mid \mathcal{D}_n;\eta^{S}_0)$}. In practice, $\eta^S_0$ is unknown and must be estimated, which yields the \emph{feasible} generalized posterior
\mbox{$q^{S}_{n,\mathrm{fe}}(\theta \mid \mathcal{D}_n)
:= q^{S}_{n}(\theta \mid \mathcal{D}_n;\hat\eta^{S})$}.


\vspace{-0.2cm}
\subsection{Algorithm}
\label{sec:algorithm}
\vspace{-0.1cm}

To compute the feasible generalized posterior in practice, we propose a generic algorithmic procedure (see Algorithm~\ref{alg:causal-gibbs}). The algorithm is agnostic to the specific causal estimand (e.g., ATE, CATE) and estimation strategy (e.g., Neyman-orthogonal estimators). The algorithm employs cross-fitting to estimate nuisance components on held-out data, which mitigates overfitting and is needed for Neyman-orthogonality later.


\begin{wrapfigure}{r}{0.50\textwidth}
\begin{minipage}{0.50\textwidth}
\hrule
\vspace{0.8ex}
\captionof{algorithm}{Generalized posteriors for causal inference}
\label{alg:causal-gibbs}
\vspace{-0.8ex}
\hrule
\vspace{0.6ex}

{\footnotesize
\textbf{Input:}  Data $\mathcal{D}_n=\{O_i\}_{i=1}^n$; estimation strategy $S$ (loss $\ell^S$ and nuisance learner); prior $\pi(\theta)$;
number of folds $K$; calibration rule $\mathsf{Calibrate}_\omega$; posterior engine (i.e., closed-form, variational inference, or MCMC sampling). \\
\textbf{Output:} $q^{S}_{n,\mathrm{fe}}$
}
\begin{algorithmic}[1]
\footnotesize
\State Split indices $\{1,\dots,n\}$ into folds $\{I_k\}_{k=1}^K$.
\For{$k=1,\dots,K$}
    \State Fit nuisances $\hat\eta^{S,(-k)}$ on $\{O_i: i\notin I_k\}$.
\EndFor 
\State Form the cross-fitted empirical loss
\vspace{-0.2cm} 
\begin{align*}
    \L^S_n(\theta;\hat\eta^S):=\frac{1}{n}\sum_{k=1}^K\sum_{i\in I_k}\ell^S \big( O_i;\theta,\hat\eta^{S,(-k)} \big) .
\end{align*}
\vspace{-0.2cm} 
\State Calibrate the parameter $\omega \leftarrow \mathsf{Calibrate}_\omega(\mathcal{D}_n,\L^S_n,\pi)$
\State Construct the generalized variational objective
\begin{align*}
    \mathcal{J}^{S}_{n,\mathrm{fe}}(q)
    := \omega n\,\E_{\theta\sim q}\!\bigl[\L^S_n(\theta;\hat\eta^S)\bigr] + \KL(q\parallel \pi).    
\end{align*}
\State Compute $q^{S}_{n,\mathrm{fe}} \in \argmin_{q\in\mathcal{Q}} \mathcal{J}^{S}_{n,\mathrm{fe}}(q)$ using the posterior engine
\State \textbf{return} $q^{S}_{n,\mathrm{fe}}$ and posterior summaries (mean, credible interval).
\end{algorithmic}
\hrule
\vspace{0.8ex}

\end{minipage}
\vspace{-0.6cm}
\end{wrapfigure}


\textbf{Choosing the calibration parameter $\omega$:}\label{sec:omega-calibration-explanation}
The calibration parameter $\omega$ is not specified by either the prior or the loss, and must be chosen separately. To achieve frequentist-valid uncertainty quantification, we follow standard practice in generalized Bayes and tune $\omega$, so that a chosen $(1-\alpha)$ posterior credible set has approximate $(1-\alpha)$ repeated-sampling coverage. Concretely, we use the bootstrap calibration scheme of \citet{syringCalibratingGeneralPosterior2019}: for a candidate $\omega$, we construct the credible region $C_{\omega,\alpha}(\mathcal{D}_n)$, estimate its coverage by nonparametric bootstrap resampling (by checking containment of the point estimate $\hat\theta=\argmin_{\theta \in \Theta} \L_{n}^{S}(\theta,\hat \eta^{S})$), and we update $\omega$ via one-dimensional search or stochastic approximation until the empirical coverage matches $1-\alpha$. We refer to \citet{syringCalibratingGeneralPosterior2019} for details.

When the object of interest is infinite-dimensional (e.g., a CATE function $\theta(\cdot)$), Bernstein--von Mises results and calibration guarantees only typically apply to finite-dimensional functionals $F(\theta)$ (such as finite set of point evaluations, the mean, the median, etc.). In these cases, we calibrate $\omega$ for the reported functional $F(\theta)$ using the same procedure. For squared-error losses, the generalized posterior corresponds to a Gaussian working likelihood with variance $\omega^{-1}$. Hence, a convenient plug-in choice is therefore $\omega=\widehat{\Var}(\hat Y^{S})^{-1}$, where $\hat Y^{S}$ denotes the pseudo-outcome. 

The choices of prior $\pi(\theta)$ and number of cross-fitting folds $K$ are discussed in Appendix~\ref{appx:prior-elicitation} and Appendix~\ref{appx:cross-fitting-folds}, respectively.

\vspace{-0.3cm}
\subsection{Theoretical results}
\label{sec:theoretical_results}
\vspace{-0.1cm}

{Our central methodological question is:} \emph{When does the feasible generalized posterior $q^{S}_{n,\mathrm{fe}}$ behave like the oracle generalized posterior
$q^{S}_{n,\mathrm{or}}$ despite the nuisance estimation?} 

The answer depends on how the nuisance estimation error enters the loss $\L^S_n(\theta;\hat\eta^S)$. State-of-the-art methods for causal ML \citep[e.g.,][]{Chernozhukov.2018, Nie.13122020, Kennedy.30042020}
address this challenge by constructing \emph{Neyman-orthogonal} losses \citep{Foster.1252019}. For such orthogonal losses, the score $D_{\theta}\L^S_n$ -- i.e., which determines the first-order optimality condition and the corresponding gradient updates -- is insensitive to local perturbations in the nuisances $D_{\eta^{S}}$ around the true value $\eta^{S}_0$, i.e.,
\begin{align}\label{def:orth-loss}
    D_{\eta^{S}}D_{\theta}\L^S_n(\theta^\star;\eta^{S}_0)[\theta-\theta^\star,\eta - \eta^{S}_0] = 0.
\end{align}
Neyman-orthogonality makes the leading bias term $\hat \eta^{S} - \eta^{S}_{0}$ vanish, so that the estimation of $\theta$ can remain $\sqrt{n}$-consistent and asymptotically normal, even when the nuisances are learned at slower, nonparametric rates (up to a product-rate condition; such as those of flexible machine learning models) \cite{Foster.1252019}.

In Theorem~\ref{thm:posterior-stability}, we show that leveraging the same orthogonality principle for constructing generalized posteriors yields an analogous result: the feasible posterior converges to the oracle even at less-than-parametric rates of convergence of the nuisances.

\begin{theorem}[Posterior stability under orthogonal losses (informal)]
\label{thm:posterior-stability}
Under the regularity conditions detailed in Appendix~\ref{Appx:proofs},
assume that:
(i) the loss is Neyman orthogonal at $(\theta^\star,\eta^S_0)$
in the sense of Eq.~\ref{def:orth-loss}; and
(ii) $\hat\eta^S$ is obtained by sample splitting/cross-fitting
and satisfies $\|\hat\eta^S-\eta^S_0\| =: r_n \to 0$.
Then, the feasible generalized posterior is asymptotically close
to the oracle one at the rate
\vspace{-0.2cm}
\begin{align}
\mathrm{TV}\!\left(q^{S}_{n,\mathrm{fe}},\ q^{S}_{n,\mathrm{or}}\right)
\;=\; O_{\prob}\!\bigl(\sqrt{n}\, r_n^{2}\bigr) \;+\; o_{\prob}(1),
\label{eq:tv-bound}
\end{align}
where $\mathrm{TV}$ denotes total variation on the target of inference: 
either (a) the full parameter $\theta\in\mathbb{R}^d$ (finite-dimensional case), or 
(b) the induced posterior of any fixed finite-dimensional projection $T(\theta)\in\mathbb{R}^m$ when $\theta$ is infinite-dimensional (e.g.\ pointwise CATE values at fixed covariate points).
In particular, if $r_n=o(n^{-1/4})$, then $\mathrm{TV}\to0$ in probability and the feasible posterior inherits the corresponding oracle BvM limit.
\end{theorem}
\vspace{-0.5cm} 
\begin{proof}
See Appendix~\ref{Appx:proofs} for formal statement and proof.
\end{proof}
\vspace{-0.2cm} 

\paragraph{Practical implications of the assumptions.}
The assumptions underlying Theorem~\ref{thm:posterior-stability} draw on established conditions from orthogonal statistical learning \citep{Foster.1252019} and generalized-posterior Bernstein--von Mises theory \citep{millerAsymptoticNormalityConcentration2021}, combining nuisance estimation requirements with regularity and posterior concentration conditions. In practice, these conditions require sufficiently accurate nuisance estimation and well-behaved losses with stable local curvature; when they fail, the posterior-stability and associated uncertainty guarantees need not hold.

The rate condition $r_n=o(n^{-1/4})$ matches the familiar threshold from orthogonal learning used in double ML theory \citep{Chernozhukov.2018}: the nuisance estimators may converge substantially slower than $n^{-1/2}$, while still yielding asymptotically valid uncertainty for low-dimensional targets. We have thus extended the theoretical guarantees of orthogonal statistical learning \cite{Foster.1252019} from point estimation to their corresponding generalized posteriors.

\begin{remark}[Why this fails without orthogonality]\label{rem:nonorth}
For non-orthogonal losses, the analogous $\mathrm{TV}$ of Eq.~\eqref{eq:tv-bound} typically scales as $O_{\prob}(\sqrt{n}\,r_n)$ rather than $O_{\prob}(\sqrt{n}\,r_n^2)$, which requires nuisance estimators to converge at near-parametric rates in order to keep posterior perturbations small. 
This explains why generalized posteriors for causal inference built from na{\"i}ve plug-in losses can be unstable in modern settings with high-dimensional and flexible ML-based nuisance estimators.
\end{remark}

\vspace{-0.1cm}
$\blacksquare$\,\textbf{Example 1 (back-door ATE).}
We illustrate the robustness of generalized posteriors by deriving an analytic solution for the back-door adjusted ATE setting using the squared-error loss constructed from the AIPW pseudo-outcome $\hat{Y}^\mathrm{AIPW}$. Specifically, let us consider the loss
\begin{align}
    \L^{\mathrm{AIPW}}_n(\theta;\hat\eta)
    = \tfrac{1}{2n}\sum_{i=1}^n \bigl(\hat{Y}^\mathrm{AIPW}(O_i;\hat\eta)-\theta\bigr)^2.
    \label{eq:ate-squared-loss}
\end{align}
The corresponding loss minimizer is the usual AIPW estimator, i.e., $\hat\theta^{\mathrm{AIPW}} = \argmin_\theta \L^{\mathrm{AIPW}}_n(\theta;\hat\eta) = \tfrac{1}{n}\sum_{i=1}^n \hat{Y}^\mathrm{AIPW}_i$. The oracle pseudo-outcome $\hat Y^\mathrm{AIPW}(\cdot,\eta)$ coincides with the efficient influence function of $\theta^{\mathrm{AIPW}}$ of the ATE, implying that the estimator $\hat\theta^{\mathrm{AIPW}}$ is asymptotically unbiased and statistically efficient. 

The corresponding feasible generalized posterior is given by
\footnotesize
\begin{align}
    q_{n,\mathrm{fe}}^{\mathrm{AIPW}}(\theta\mid\mathcal{D}_n)
    &\propto
    \exp\!\left\{-\omega n\,\L^{\mathrm{AIPW}}_n(\theta;\hat\eta)\right\}\pi(\theta)
    =
    \exp\!\left\{-\tfrac{\omega}{2}\sum_{i=1}^n(\hat Y^\mathrm{AIPW}_i-\theta)^2\right\}\pi(\theta) ,
    \label{eq:ate-gibbs-conjugate}
\end{align}
\normalsize
which allows for a closed-form solution under a normal prior $\pi(\theta)=\mathcal{N}(m_0,s_0^2)$. In that case, the posterior is exactly normal, i.e.,
\footnotesize
\begin{align*}
    q_{n,\mathrm{fe}}^{\mathrm{AIPW}}(\cdot \mid\mathcal{D}_n) &= \mathcal{N}(m_p,s_p^2), \ \text{where} \quad
    s_p^2 = \bigl(s_0^{-2}+\omega n\bigr)^{-1}, \
    m_p = s_p^2\bigl(s_0^{-2}m_0+\omega n\,\hat \theta^{\mathrm{AIPW}} \bigr).
\end{align*}
\normalsize
In particular, under a diffuse prior $s_0^2\to\infty$, we have $m_p\to \hat\theta^{\mathrm{AIPW}}$ and $s_p^2\to (\omega n)^{-1}$, meaning the generalized posterior centers exactly at the classical
orthogonal point estimate (while $\omega$ controls its spread). 

Importantly, the only data-dependent term of the analytical solution of the posterior parameters $m_p,s_p^2$ is the point-estimate $\hat\theta^{\mathrm{AIPW}}$, i.e., the minimizer of $\L^{\mathrm{AIPW}}_n(\theta;\hat\eta)$. As such, we can see how Neyman-orthogonality of the loss directly translates into convergence guarantees for the generalized posterior under nuisance estimation error. Our theorem shows this generally, even when no analytical solution to the posterior is available. 

Finally, note that the analytical solution to the posterior was achieved by recognizing that Eq.~\eqref{eq:ate-gibbs-conjugate} is algebraically equivalent to having a Gaussian ``working likelihood'' of the pseudo-outcomes $\hat Y^\mathrm{AIPW}_i \mid \theta \sim \mathcal{N}(\theta, \omega^{-1})$ and employing the normal-normal conjugacy. Under this interpretation, the generalized posterior is calibrated by setting $\omega^{-1} = \Var(\hat Y^{\mathrm{AIPW}})$. Since $\hat \theta^{\mathrm{AIPW}}$ is the efficient estimator of $\theta_{\mathrm{ATE}}$, it achieves the lowest possible asymptotic variance among all regular, asymptotically linear (RAL) estimators. Consequently, 
\begin{align}
    s_p^2 \to (n\omega)^{-1} = n^{-1}\Var(\hat Y^{\mathrm{AIPW}}) = \Var(\hat\theta_{\mathrm{ATE}}),
\end{align}
which implies that, under a diffuse prior, the corresponding frequentist-calibrated generalized posterior achieves the smallest achievable frequentist-calibrated variance and has, in turn, the narrowest faithful credible intervals. 


\vspace{-0.3cm}
\section{Experiments}
\label{sec:experiments}
\vspace{-0.2cm}
\FloatBarrier

In our experiments, we investigate how generalized posteriors perform across different causal estimation strategies, focusing on (i) faithful and efficient uncertainty quantification, (ii) robustness to model misspecification, and (iii) sensitivity to calibration and prior information. We present results for the back-door ATE setting, while CATE estimation and additional experiments are reported in Appendix~\ref{app:other-experiments}.

\begin{table*}[h]
\centering
\resizebox{\textwidth}{!}{%
\begin{tabular}{lcccccccccc}
\hline
Strategy & Orthogonal & $\mathcal{D}_{1}$ & $\mathcal{D}_{2}$ & $\mathcal{D}_{3}$ & $\mathcal{D}_{4}$ & $\mathcal{D}_{5}$ & $\mathcal{D}_{6}$ & $\mathcal{D}_{7}$ & $\mathcal{D}_{8}$ & $\mathcal{D}_{9}$ \\ 
\hline
RA & \xmark & $0.580$ \scriptsize (0.432-0.718) & $0.520$ \scriptsize (0.374-0.663) & $0.480$ \scriptsize (0.337-0.626) & $0.680$ \scriptsize (0.533-0.805) & $0.580$ \scriptsize (0.432-0.718) & $0.080$ \scriptsize (0.022-0.192) & $0.300$ \scriptsize (0.179-0.446) & $0.460$ \scriptsize (0.318-0.607) & $0.420$ \scriptsize (0.282-0.568) \\
IPW & \xmark & $1.000$ \scriptsize (0.929-1.000) & $1.000$ \scriptsize (0.929-1.000) & $1.000$ \scriptsize (0.929-1.000) & $1.000$ \scriptsize (0.929-1.000) & $1.000$ \scriptsize (0.929-1.000) & $0.420$ \scriptsize (0.282-0.568) & $0.980$ \scriptsize (0.894-0.999) & $1.000$ \scriptsize (0.929-1.000) & $1.000$ \scriptsize (0.929-1.000) \\ 
AIPW & \cmark & $\textbf{0.940}$ \scriptsize (0.835-0.987) & $\textbf{0.980}$ \scriptsize (0.894-0.999) & $\textbf{0.980}$ \scriptsize (0.894-0.999) & $\textbf{0.920}$ \scriptsize (0.808-0.978) & $\textbf{0.940}$ \scriptsize (0.835-0.987) & $\textbf{0.580}$ \scriptsize (0.432-0.718) & $\textbf{0.940}$ \scriptsize (0.835-0.987) & $\textbf{0.980}$ \scriptsize (0.894-0.999) & $\textbf{0.940}$ \scriptsize (0.835-0.987) \\ 
\hline
\end{tabular}
 
}
\caption{
\scriptsize{
\textbf{Coverage for back-door ATE.} The $95\%$ $\mathrm{CrI}^{(r)}_{0.95}$ credible interval's coverage $\hat C$ (and its $95\%$ confidence interval) for each strategy $S \in \{\mathrm{AIPW,IPW,RA}\}$ across $9$ synthetic datasets $\{\mathcal{D}_1,\ldots,\mathcal{D}_9\}$. Results are computed from $R=50$ repetitions at sample size $n=1000$. Closest to exact frequentist calibration at $95\%$ is in \textbf{bold}.
}
}
\label{table:ate-coverage-analytic}
\end{table*}

\begin{wrapfigure}{r}{0.60\textwidth}
\vspace{-0.4cm}
\centering
\includegraphics[width=\linewidth]{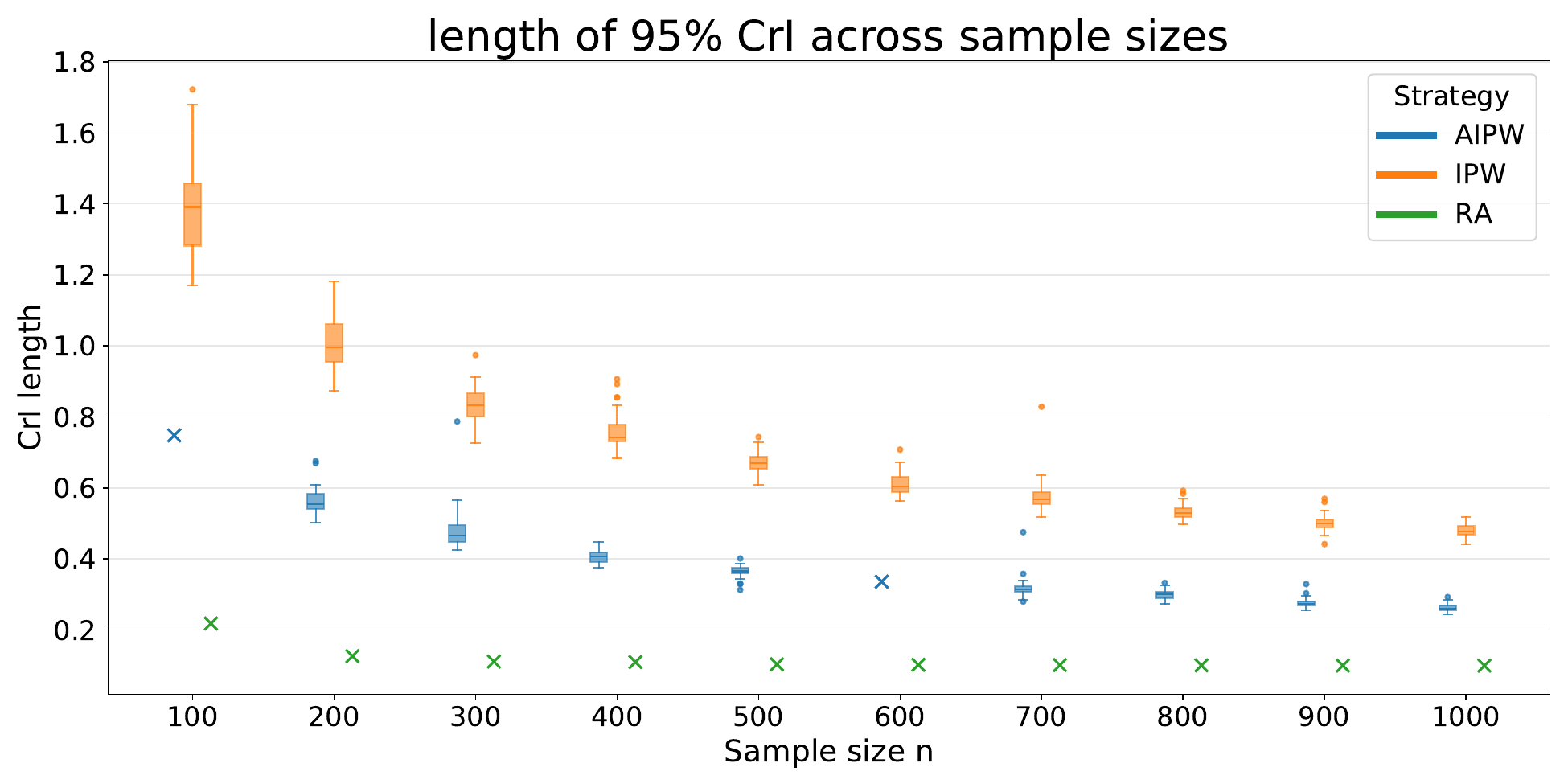}
\vspace{-0.35cm}
\caption{
\scriptsize{
\textbf{CrI length for back-door ATE.} Box plots of the lengths of the $95\%$ $\mathrm{CrI}^{(r)}_{0.95}$ credible interval across $R=50$ repetitions for each strategy $S \in \{\mathrm{AIPW,IPW,RA}\}$, the dataset $\mathcal{D}_1$, and a varying sample size $n \in \{100,\ldots,1000\}$. Unfaithful credible intervals are reported with a cross at their median length.
}
}
\label{fig:ate-length-convergence}
\vspace{-0.5cm}
\end{wrapfigure}

\textbf{Settings:}
We consider nine synthetic back-door settings $\{\mathcal{D}_1,\ldots,\mathcal{D}_9\}$, each with a known ground-truth ATE (see Appendix~\ref{app:dgp-backdoor-description}). We first compare generalized posteriors constructed using three standard estimation strategies $S\in\{\mathrm{RA,IPW,AIPW}\}$, with their corresponding nuisance components estimated using flexible ML models. For each strategy and dataset, we repeat the posterior estimation pipeline over $R=50$ random seeds at sample size $n=1000$. The subsequent robustness and sensitivity experiments are described alongside their respective results, with implementation details provided in Appendix~\ref{app:implementation}.

\textbf{Metrics:}
We measure: (i)~\textit{Coverage:} For each repetition $r=1,\ldots,R$, we construct a nominal $95\%$ credible interval $\mathrm{CrI}^{(r)}_{0.95}$ and estimate its frequentist coverage as $\hat C = \tfrac{1}{R}\sum_{r=1}^{R}\ind\{\theta^\star \in \mathrm{CrI}^{(r)}_{0.95}\}$.
(ii)~\textit{Interval length:} We measure the average credible interval length across repetitions. To avoid favoring artificially narrow intervals, we mark as unfaithful those strategy-dataset pairs for which the upper bound of the $95\%$ confidence interval for $\hat C$ is below $95\%$, and exclude them when comparing interval lengths.
(iii)~\textit{Mean squared error (MSE):} For comparisons of point-estimation accuracy, we report the empirical MSE of the posterior mean relative to the true causal estimand $\theta^\star$.

\begin{wraptable}{r}{0.60\textwidth}
\centering
\resizebox{\linewidth}{!}{%
\begin{tabular}{lcccc}
\hline
 & \multicolumn{2}{c}{Well-specified outcome model} & \multicolumn{2}{c}{Misspecified outcome model} \\
Method & Coverage $\hat C$ & MSE & Coverage $\hat C$ & MSE \\
\hline
Standard Bayes & $0.932$ \scriptsize (0.906-0.952) & $0.0050$ & $0.006$ \scriptsize (0.001-0.017) & $0.1572$ \\
Generalized Bayes ($\mathrm{AIPW}$) & $0.932$ \scriptsize (0.906-0.952) & $0.0053$ & $0.968$ \scriptsize (0.949-0.982) & $0.0227$ \\
\hline
\end{tabular}
}
\caption{
\scriptsize{
\textbf{Standard vs.\ generalized Bayes under outcome misspecification.}
Empirical coverage of nominal $95\%$ credible intervals
(with $95\%$ confidence intervals) and posterior-mean MSE
under well-specified and misspecified outcome models.
Results are based on $R=500$ repetitions with $n=1000$,
using a diffuse $\mathcal N(0,100^2)$ prior.
}
}
\label{table:e1_standard_vs_generalized}
\vspace{-0.3cm}
\end{wraptable}

\textbf{Coverage and efficiency across estimation strategies:}
Table~\ref{table:ate-coverage-analytic} presents the empirical coverage in each setting. Consistent with our theoretical results, the Neyman-orthogonal AIPW-based generalized posterior generally achieves the most faithful uncertainty quantification. Figure~\ref{fig:ate-length-convergence} further illustrates the efficiency of the AIPW-based generalized posterior as sample size increases. Detailed results on credible interval lengths are provided in Appendix~\ref{appx:ate-length-exp}.

\textbf{Robustness to model misspecification:}
Table~\ref{table:e1_standard_vs_generalized} compares likelihood-based Bayesian inference with our AIPW-based generalized posterior under correctly specified and partially misspecified outcome models. While both approaches perform similarly under correct specification, the generalized posterior retains near-nominal coverage under partial misspecification, illustrating how the double robustness of the AIPW loss translates into robust uncertainty quantification when the propensity model remains correctly specified.



\begin{figure}[t]
\centering

\begin{minipage}[t]{0.48\linewidth}
\vspace{0pt}
\centering
\includegraphics[width=\linewidth]{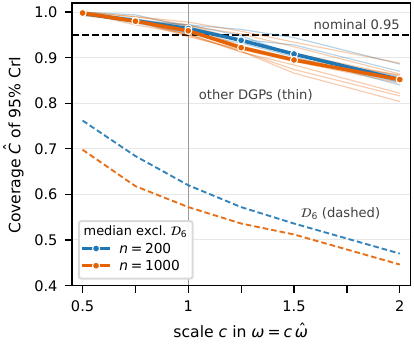}
\caption{
\scriptsize{
\textbf{Sensitivity to the calibration parameter $\omega$.}
Coverage $\hat C$ of the $95\%$ credible interval of the
$\mathrm{AIPW}$ generalized posterior when the plug-in
calibration $\hat\omega=\widehat{\Var}(\hat Y^{\mathrm{AIPW}})^{-1}$
is rescaled to $\omega=c\,\hat\omega$, across
$\{\mathcal{D}_1,\ldots,\mathcal{D}_9\}$ at
$n\in\{200,1000\}$ ($R=500$ repetitions).
Bold: median across datasets excluding $\mathcal{D}_6$;
thin: individual datasets; dashed: $\mathcal{D}_6$
(limited overlap), which under-covers for every $c$.
}
}
\label{fig:e10_omega_scale_sweep}
\end{minipage}
\hfill
\begin{minipage}[t]{0.48\linewidth}
\vspace{0pt}
\centering
\includegraphics[width=\linewidth]{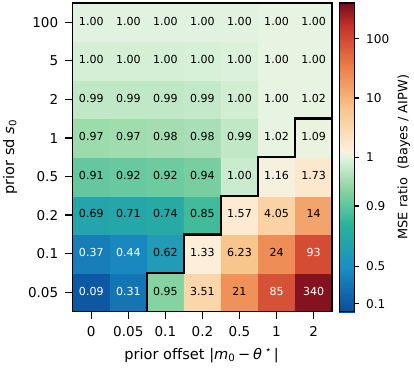}
\caption{
\scriptsize{
\textbf{Prior strength: when does the prior help?}
Ratio of the MSE of the generalized posterior mean to
the MSE of the frequentist $\hat\theta^{\mathrm{AIPW}}$
on the same data, for priors $\mathcal{N}(m_0,s_0^2)$
with offset $|m_0-\theta^\star|$ and scale $s_0$
($\mathcal{D}_1$, $n=1000$, $R=500$ per cell).
Green: the prior improves MSE (ratio $<1$);
red: it hurts. The black line marks the asymptotic
boundary $|m_0-\theta^\star|=\sqrt{2}\,s_0$.
}
}
\label{fig:e3_mse_ratio_heatmap}
\end{minipage}

\end{figure}

\textbf{Sensitivity to posterior calibration:}
Figure~\ref{fig:e10_omega_scale_sweep} examines the sensitivity of frequentist coverage to the choice of calibration parameter $\omega$. The results highlight the importance of appropriate calibration for faithful uncertainty quantification, while also showing that calibration alone cannot remedy persistent undercoverage in challenging settings such as limited overlap (i.e. setting $\mathcal{D}_6$). Analytic and bootstrap calibration yield comparable results, with further details in Appendix~\ref{app:e10-calibration-comparison}.

\textbf{The role of prior information:}
To gain intuition about when prior information can improve causal effect estimation, consider a simple setting where the generalized posterior is available analytically. Suppose $\hat\theta$ is an unbiased estimator of $\theta^\star$ with variance $v_n$, and we place a Gaussian prior $\pi(\theta)=\mathcal N(m_0,s_0^2)$ directly on the causal estimand. With $\omega$ calibrated to $v_n$, we can show that the generalized posterior mean achieves lower frequentist MSE than $\hat\theta$ if and only if
\begin{equation}
(m_0-\theta^\star)^2 < 2s_0^2 + v_n.
\end{equation}
This illustrates how incorporating prior information can improve estimation accuracy when the prior is sufficiently well-centered, while a strongly informative but poorly centered prior may instead be harmful. Figure~\ref{fig:e3_mse_ratio_heatmap} illustrates this trade-off empirically across different prior specifications.

\textbf{Conclusion:}
We have presented a generalized Bayesian framework for causal inference that places priors directly on causal estimands and updates them using identification-driven losses, without requiring a likelihood model for the full data-generating process. For Neyman-orthogonal losses, our theoretical results establish robustness to nuisance estimation error and asymptotically valid uncertainty quantification under suitable regularity and calibration conditions. Our experiments illustrate these properties, while also highlighting the importance of calibration and the challenges posed by limited overlap. Further experiments, including CATE estimation and sensitivity analyses, are reported in Appendix~\ref{app:other-experiments}, with additional practical considerations in Appendix~\ref{appx:additional-details}. The framework remains subject to causal identification assumptions and finite-sample limitations, discussed further in Appendix~\ref{appx:limitations}.

\newpage

\begin{ack}


This paper is supported by the DAAD programme Konrad Zuse Schools of Excellence in Artificial Intelligence, sponsored by the Federal Ministry of Research, Technology and Space. Additionally, this work has been supported by the German Federal Ministry of Education and Research (Grant: 01IS24082). We also thank Valentyn Melnychuk for valuable comments and suggestions on earlier versions of this manuscript.
\end{ack}



\bibliography{references1}

\newpage
\appendix

\section{Proofs}
\label{Appx:proofs}

We provide the formal statement and proof of Theorem~\ref{thm:posterior-stability}. We consider
(a) finite-dimensional $\theta$, and (b) finite-dimensional projections of an
infinite-dimensional $\theta$. In the finite-dimensional case, we first derive the feasible
Bernstein--von Mises (BvM) approximation from regularity conditions on the cross-fitted loss,
then compare it with the oracle posterior.

\subsection{Finite dimensional \texorpdfstring{$\theta$}{theta}}

\begin{theorem}[Posterior stability under orthogonal losses (finite-dimensional $\theta$)]
\label{thm:posterior-stability-finite}
Let $\Theta\subset\R^d$ be open, and let $\eta\in\mathcal H$ be a nuisance parameter with truth
$\eta_0$. For a loss $\ell(O;\theta,\eta)$, define the empirical and population risks
\begin{equation}
\L_n(\theta;\eta_0):=\frac1n\sum_{i=1}^n\ell(O_i;\theta,\eta_0),
\qquad L(\theta;\eta):=\E_{\P}\{\ell(O;\theta,\eta)\}.
\end{equation}
Let $\theta^\star$ be the unique minimizer of $L(\cdot;\eta_0)$, and suppose that
$V_0:=\nabla_{\theta\theta}L(\theta^\star;\eta_0)$ is positive definite.
For a fixed $K$, partition the observations into validation folds $I_1,\ldots,I_K$,
and let $\hat\eta^{(-k)}$ be fitted without the observations in $I_k$.
Define the cross-fitted feasible risk
\begin{equation}
\L_n(\theta;\hat\eta)
:=\frac1n\sum_{k=1}^K\sum_{i\in I_k}
\ell(O_i;\theta,\hat\eta^{(-k)}).
\end{equation}
Fix $\omega>0$ and a proper prior density $\pi$ on $\Theta$. Define
\begin{equation}
q_{n,\mathrm{or}}(\theta\mid\mathcal D_n)\ \propto\
\pi(\theta)\exp\{-\omega n\L_n(\theta;\eta_0)\},
\qquad
q_{n,\mathrm{fe}}(\theta\mid\mathcal D_n)\ \propto\
\pi(\theta)\exp\{-\omega n\L_n(\theta;\hat\eta)\},
\end{equation}
and their empirical minimizers
\begin{equation}
\hat\theta_{\mathrm{or}}:=\arg\min_{\theta\in\Theta}\L_n(\theta;\eta_0),
\qquad
\hat\theta_{\mathrm{fe}}:=\arg\min_{\theta\in\Theta}\L_n(\theta;\hat\eta).
\end{equation}
Let $r_n\downarrow0$ be a deterministic nuisance-rate sequence, and let
$B_0:=B(\theta^\star,\delta_0)$ for a fixed $\delta_0>0$ with
$\overline B_0\subset\Theta$. Assume:
\begin{enumerate}
\item[(A1)] \textbf{(Oracle BvM).}
The oracle generalized posterior satisfies
\begin{equation}\label{eq:BvM-or-assmp}
\mathrm{TV}\!\Bigl(q_{n,\mathrm{or}}(\cdot\mid\mathcal D_n),\
\mathcal N\!\bigl(\hat\theta_{\mathrm{or}},(n\omega V_0)^{-1}\bigr)\Bigr)
=o_{\P}(1).
\end{equation}
This is the oracle BvM input; no BvM approximation for the feasible posterior is assumed.

\item[(A2)] \textbf{(Neyman orthogonality).}
The loss is Neyman orthogonal at $(\theta^\star,\eta_0)$ in the sense of
Eq.~\eqref{def:orth-loss}.

\item[(A3)] \textbf{(Orthogonal-score regularity).}
Write $\psi(O;\theta,\eta):=\nabla_\theta\ell(O;\theta,\eta)$ and
$\Delta\psi_k(O):=\psi(O;\theta^\star,\hat\eta^{(-k)})-
\psi(O;\theta^\star,\eta_0)$. Assume
$\|\hat\theta_{\mathrm{or}}-\theta^\star\|=O_{\P}(n^{-1/2})$ and, uniformly over
$k\in\{1,\ldots,K\}$,
\begin{align}
\big\|\E_{\P}[\Delta\psi_k(O)]\big\|&=O_{\P}(r_n^2),
&\E_{\P}\big[\|\Delta\psi_k(O)\|^2\big]&=O_{\P}(r_n^2),
\label{eq:orth-score-moments}
\end{align}
where the expectations are conditional on the corresponding fold's training sample.
Further assume the local Hessian perturbation bound
\begin{equation}\label{eq:hessian-nuisance-stability}
\sup_{\theta\in B_0}\big\|
\nabla_{\theta\theta}\L_n(\theta;\hat\eta)
-\nabla_{\theta\theta}\L_n(\theta;\eta_0)
\big\|=O_{\P}(r_n).
\end{equation}
These are standard types of second-order bias, stochastic equicontinuity, and
local smoothness controls used in orthogonal M-estimation; they make the center
stability argument below explicit.

\item[(A4)] \textbf{(Regularity of the cross-fitted feasible risk).}
The following conditions hold for the same $B_0$:
\begin{enumerate}
\item[(C1)] \textbf{Prior regularity.}
$\pi$ is continuous and strictly positive at $\theta^\star$.

\item[(C2)] \textbf{Local smoothness.}
For $\eta$ in a neighborhood $\mathcal N$ of $\eta_0$, the map
$\theta\mapsto\ell(O;\theta,\eta)$ is three times continuously differentiable
on a neighborhood of $\overline B_0$. There are integrable envelopes
$M_j(O)$ such that
\begin{equation}
\sup_{\theta\in\overline B_0,\ \eta\in\mathcal N}
\|\nabla_\theta^j\ell(O;\theta,\eta)\|\le M_j(O),
\qquad \E_{\P}[M_j(O)]<\infty,\quad j=0,1,2,3.
\end{equation}

\item[(C3)] \textbf{Cross-fitting and nuisance consistency.}
$K$ is fixed, $|I_k|/n\to\alpha_k\in(0,1)$, and
$\hat\eta^{(-k)}$ uses only $\{O_i:i\notin I_k\}$. With probability
approaching one, $\hat\eta^{(-k)}\in\mathcal N$ for each $k$, and
\begin{equation}
\max_{1\le k\le K}\|\hat\eta^{(-k)}-\eta_0\|=O_{\P}(r_n),
\qquad r_n\to0.
\end{equation}

\item[(C4)] \textbf{Identification and local curvature.}
$\theta^\star$ is the unique minimizer of $L(\cdot;\eta_0)$.
For constants $\lambda_0,c_{\mathrm{sep}}>0$,
\begin{align}
\nabla_{\theta\theta}L(\theta;\eta_0)&\succeq 2\lambda_0 I_d
&& \text{for all }\theta\in\overline B_0,\nonumber\\
\inf_{\theta\notin B_0}
\big\{L(\theta;\eta_0)-L(\theta^\star;\eta_0)\big\}
&\ge c_{\mathrm{sep}}.
\end{align}

\item[(C5)] \textbf{Continuity in the nuisance.}
The maps $\eta\mapsto L(\theta;\eta)$ and
$\eta\mapsto\nabla_{\theta\theta}L(\theta;\eta)$ are continuous at $\eta_0$
uniformly over $\theta\in\overline B_0$.

\item[(C6)] \textbf{Global empirical separation.}
With probability approaching one,
\begin{equation}\label{eq:feasible-global-separation}
\inf_{\theta\notin B_0}
\big\{\L_n(\theta;\hat\eta)-\L_n(\theta^\star;\hat\eta)\big\}
\ge\frac{c_{\mathrm{sep}}}{2}.
\end{equation}
For example, this follows from (C4) and a suitable global uniform
law of large numbers for the cross-fitted feasible risk. It is stated
separately because a uniform law only on $B_0$ is insufficient to
control posterior mass outside $B_0$.
\end{enumerate}
\end{enumerate}

Then the feasible generalized posterior admits the BvM approximation
\begin{equation}\label{eq:BvM-fe-assmp}
\mathrm{TV}\!\Bigl(q_{n,\mathrm{fe}}(\cdot\mid\mathcal D_n),\
\mathcal N\!\bigl(\hat\theta_{\mathrm{fe}},(n\omega V_0)^{-1}\bigr)\Bigr)
=o_{\P}(1),
\end{equation}
and the posterior-stability bound is
\begin{align}
\mathrm{TV}\!\left(q_{n,\mathrm{fe}}(\cdot\mid\mathcal D_n),\
q_{n,\mathrm{or}}(\cdot\mid\mathcal D_n)\right)
=O_{\P}\!\bigl(\sqrt n\,r_n^2\bigr)+o_{\P}(1).
\label{eq:tv-bound-finite}
\end{align}
In particular, if $r_n=o(n^{-1/4})$, then
$\mathrm{TV}(q_{n,\mathrm{fe}},q_{n,\mathrm{or}})\to0$ in probability
and the feasible posterior inherits the oracle BvM limit.
\end{theorem}

\paragraph{Discussion of assumptions.}
Our assumptions draw on two established strands of theory: Neyman-orthogonal statistical learning \citep{Foster.1252019}, which motivates (A2)--(A3), and generalized-posterior Bernstein--von Mises (BvM) theory \citep{millerAsymptoticNormalityConcentration2021}, which motivates (A1) and (A4). Specifically, (A1) provides the oracle BvM approximation as a benchmark, while (A2)--(A3) ensure that Neyman-orthogonality translates into second-order stability under nuisance estimation error, accounting also for the stochastic fluctuations introduced by estimated nuisances. Assumption (A4) provides sufficient regularity conditions to derive the feasible BvM approximation rather than assuming it directly. Conditions (C1)--(C5) ensure prior regularity, smoothness, nuisance consistency, and local curvature of the cross-fitted empirical risk, allowing for a local quadratic approximation of the feasible posterior. Condition (C6) additionally ensures global empirical separation, which is needed to control posterior mass outside the neighborhood where the local approximation holds. Together, these conditions allow us to establish posterior stability through a comparison of the oracle and feasible Gaussian approximations.

\begin{proof}[Proof of Theorem~\ref{thm:posterior-stability-finite}]
\label{proof:finite-dim}
We omit conditioning on $\mathcal D_n$ when it is unambiguous. The proof
first establishes Eq.~\eqref{eq:BvM-fe-assmp} from (C1)--(C6), then uses
orthogonal-score stability to compare the posterior centers.

\paragraph{Step 1 (Local uniform laws for the cross-fitted risk).}
For $j\in\{0,2\}$, write
$\varphi_j(O;\theta,\eta):=\nabla_\theta^j\ell(O;\theta,\eta)$,
with $\varphi_0=\ell$.
Fix a fold $k$, and let $\mathcal T_k$ be the sigma-field generated by its
training data and any algorithmic randomness used for $\hat\eta^{(-k)}$.
Conditional on $\mathcal T_k$, the observations in $I_k$ are i.i.d. and
independent of $\hat\eta^{(-k)}$. Using (C2) and a conditional uniform law
of large numbers, for $j\in\{0,2\}$,
\begin{equation}\label{eq:fold-ulln}
\sup_{\theta\in B_0}\left\|
\frac1{|I_k|}\sum_{i\in I_k}
\varphi_j(O_i;\theta,\hat\eta^{(-k)})
-\E_{\P}\big[\varphi_j(O;\theta,\hat\eta^{(-k)})\big]
\right\|=o_{\P}(1).
\end{equation}
By (C3) and (C5), the corresponding expectations converge uniformly
to their oracle counterparts. Summing over the fixed number of folds gives
\begin{align}
\sup_{\theta\in B_0}\big|\L_n(\theta;\hat\eta)-L(\theta;\eta_0)\big|
&=o_{\P}(1),\label{eq:ulln-risk}\\
\sup_{\theta\in B_0}\big\|
\nabla_{\theta\theta}\L_n(\theta;\hat\eta)
-\nabla_{\theta\theta}L(\theta;\eta_0)\big\|
&=o_{\P}(1).\label{eq:ulln-hess}
\end{align}
Moreover, the third-derivative envelope in (C2) implies
\begin{equation}\label{eq:third-deriv}
\sup_{\theta\in B_0}\big\|
\nabla_{\theta\theta\theta}\L_n(\theta;\hat\eta)\big\|
\le \frac1n\sum_{i=1}^n M_3(O_i)=O_{\P}(1).
\end{equation}

\paragraph{Step 2 (Consistency and local curvature of the feasible minimizer).}
The global separation condition (C6) places any global minimizer of
$\L_n(\cdot;\hat\eta)$ inside $B_0$ with probability tending to one.
On $B_0$, Eq.~\eqref{eq:ulln-risk}, population identification in (C4),
and the argmin theorem give $\hat\theta_{\mathrm{fe}}\to\theta^\star$
in probability. By (C4) and Eq.~\eqref{eq:ulln-hess}, with probability
tending to one,
\begin{equation}\label{eq:feasible-local-curvature}
\nabla_{\theta\theta}\L_n(\theta;\hat\eta)
\succeq\lambda_0 I_d\qquad\text{for every }\theta\in B_0.
\end{equation}
Also, the consistency of $\hat\theta_{\mathrm{fe}}$ and
Eq.~\eqref{eq:ulln-hess} imply
\begin{equation}\label{eq:hess-at-hat}
\nabla_{\theta\theta}\L_n(\hat\theta_{\mathrm{fe}};\hat\eta)
=V_0+o_{\P}(1).
\end{equation}

\paragraph{Step 3 (Local quadratic expansion).}
Let $h:=\sqrt n(\theta-\hat\theta_{\mathrm{fe}})$ and
$\theta(h):=\hat\theta_{\mathrm{fe}}+h/\sqrt n$.
Let $Q_n$ be the distribution of $h$ induced by $q_{n,\mathrm{fe}}$ and set
$\Phi:=\mathcal N(0,(\omega V_0)^{-1})$.
The unnormalized density of $Q_n$, up to a constant independent of $h$, is
\begin{equation}\label{eq:local-unnormalized}
\tilde q_n(h):=\pi(\theta(h))
\exp\!\left\{-\omega n
\bigl(\L_n(\theta(h);\hat\eta)
-\L_n(\hat\theta_{\mathrm{fe}};\hat\eta)\bigr)\right\}.
\end{equation}
Since the feasible minimizer is interior with high probability,
$\nabla_\theta\L_n(\hat\theta_{\mathrm{fe}};\hat\eta)=0$.
Taylor expansion and Eq.~\eqref{eq:hess-at-hat} yield, for each fixed $M<\infty$,
\begin{equation}\label{eq:quad-approx}
\sup_{\|h\|\le M}\left|
n\bigl(\L_n(\theta(h);\hat\eta)
-\L_n(\hat\theta_{\mathrm{fe}};\hat\eta)\bigr)
-\frac12 h^\top V_0 h\right|=o_{\P}(1).
\end{equation}
By (C1) and $\hat\theta_{\mathrm{fe}}\to\theta^\star$,
\begin{equation}\label{eq:prior-ratio}
\sup_{\|h\|\le M}\left|
\frac{\pi(\theta(h))}{\pi(\hat\theta_{\mathrm{fe}})}-1
\right|=o_{\P}(1).
\end{equation}

\paragraph{Step 4 (Posterior tail control).}
For clarity, write
$\Delta\L_n(\theta):=\L_n(\theta;\hat\eta)
-\L_n(\hat\theta_{\mathrm{fe}};\hat\eta)$.
Choose $\delta_1>0$ small enough that
$\overline{B(\theta^\star,3\delta_1)}\subset B_0$ and
$\pi$ is bounded on $B(\theta^\star,3\delta_1)$.
With high probability, $\hat\theta_{\mathrm{fe}}\in
B(\theta^\star,\delta_1)$. The curvature bound
\eqref{eq:feasible-local-curvature} then gives
\begin{equation}\label{eq:local-quad-lower}
\Delta\L_n(\theta)\ge
\frac{\lambda_0}{2}\|\theta-\hat\theta_{\mathrm{fe}}\|^2
\qquad\text{for all }\theta\in B_0.
\end{equation}
The local expansion in Step 3 and positivity of $\pi$ near $\theta^\star$
provide a lower bound for the normalizing integral:
\begin{equation}\label{eq:normalizer-lower}
\int_\Theta\pi(\theta)
\exp\{-\omega n\Delta\L_n(\theta)\}\,d\theta
\ge c\,n^{-d/2}
\qquad\text{with probability tending to one},
\end{equation}
for a fixed $c>0$.

For $M<\sqrt n\,\delta_1$, split the tail region
$\{\|h\|>M\}$ into (i) points with
$M/\sqrt n<\|\theta-\hat\theta_{\mathrm{fe}}\|\le\delta_1$,
(ii) points in $B_0$ farther than $\delta_1$ from the minimizer,
and (iii) points outside $B_0$.
In region (i), the boundedness of $\pi$ and
Eq.~\eqref{eq:local-quad-lower} yield a contribution to the
normalizing integral bounded by
\begin{equation}
C n^{-d/2}\int_{\|h\|>M}
\exp\!\left\{-\frac{\omega\lambda_0}{2}\|h\|^2\right\}\,dh.
\end{equation}
In region (ii), Eq.~\eqref{eq:local-quad-lower}
provides an exponential bound $\exp(-\omega n\lambda_0\delta_1^2/2)$.
In region (iii), (C6) and the minimizing property of
$\hat\theta_{\mathrm{fe}}$ give
$\Delta\L_n(\theta)\ge c_{\mathrm{sep}}/2$.
Since $\pi$ is a proper density, the total integrals over (ii)
and (iii) are exponentially small. Dividing by
Eq.~\eqref{eq:normalizer-lower}, we conclude that, for constants
$C,c'>0$ and all sufficiently large $n$,
\begin{equation}\label{eq:tails}
Q_n(\{\|h\|>M\})\le
C\int_{\|h\|>M}e^{-\omega\lambda_0\|h\|^2/2}\,dh
+C n^{d/2}e^{-c'n}
\quad\text{with probability tending to one}.
\end{equation}
Thus the local laws $Q_n$ are tight in probability.

\paragraph{Step 5 (Feasible BvM).}
Write $g(h):=\exp\{-\omega h^\top V_0h/2\}$ and let
$\phi(h):=g(h)/\int_{\R^d}g(u)\,du$ be the density of $\Phi$.
Equations~\eqref{eq:quad-approx} and \eqref{eq:prior-ratio} imply
\begin{equation}
\sup_{\|h\|\le M}\left|
\frac{\tilde q_n(h)}{\pi(\hat\theta_{\mathrm{fe}})g(h)}-1
\right|=o_{\P}(1)
\qquad\text{for every fixed }M<\infty.
\end{equation}
Combining this uniform local approximation with
Eq.~\eqref{eq:tails} gives convergence of the normalized density
of $Q_n$ to $\phi$ in $L^1$, and hence
\begin{equation}\label{eq:BvM-fe}
\mathrm{TV}\!\Bigl(q_{n,\mathrm{fe}},
\mathcal N\!\bigl(\hat\theta_{\mathrm{fe}},(n\omega V_0)^{-1}\bigr)\Bigr)
=o_{\P}(1),
\end{equation}
which proves Eq.~\eqref{eq:BvM-fe-assmp} without assuming it.

\paragraph{Step 6 (Orthogonality and the center shift).}
Set
$R_n(\theta):=\nabla_\theta\L_n(\theta;\hat\eta)
-\nabla_\theta\L_n(\theta;\eta_0)$.
For the fold-specific score difference $\Delta\psi_k$ in (A3),
conditional on the training sample the held-out observations are i.i.d.
By the bias and second-moment bounds in Eq.~\eqref{eq:orth-score-moments}
and fixed $K$,
\begin{equation}\label{eq:score-at-truth}
\|R_n(\theta^\star)\|
=O_{\P}\!\left(r_n^2+\frac{r_n}{\sqrt n}\right).
\end{equation}
The local Hessian bound \eqref{eq:hessian-nuisance-stability} and
$\|\hat\theta_{\mathrm{or}}-\theta^\star\|=O_{\P}(n^{-1/2})$ imply
\begin{equation}
\|R_n(\hat\theta_{\mathrm{or}})-R_n(\theta^\star)\|
=O_{\P}(r_n n^{-1/2}).
\end{equation}
Both minimizers lie in $B_0$ with probability tending to one.
Using the first-order conditions and the mean-value formula along
the segment joining them, together with
Eq.~\eqref{eq:feasible-local-curvature}, gives
\begin{equation}\label{eq:center-shift}
\|\hat\theta_{\mathrm{fe}}-\hat\theta_{\mathrm{or}}\|
=O_{\P}\!\left(r_n^2+\frac{r_n}{\sqrt n}\right).
\end{equation}
The $r_n/\sqrt n$ term is the usual cross-fitted empirical fluctuation;
it is negligible at the posterior scale since $r_n\to0$.

\paragraph{Step 7 (Gaussian comparison and posterior stability).}
Define
\begin{equation}
G_{\mathrm{or}}:=\mathcal N\!\bigl(\hat\theta_{\mathrm{or}},(n\omega V_0)^{-1}\bigr),
\qquad
G_{\mathrm{fe}}:=\mathcal N\!\bigl(\hat\theta_{\mathrm{fe}},(n\omega V_0)^{-1}\bigr).
\end{equation}
By the triangle inequality,
\begin{align}
\mathrm{TV}(q_{n,\mathrm{fe}},q_{n,\mathrm{or}})
&\le \mathrm{TV}(q_{n,\mathrm{fe}},G_{\mathrm{fe}})
+\mathrm{TV}(G_{\mathrm{fe}},G_{\mathrm{or}})
+\mathrm{TV}(G_{\mathrm{or}},q_{n,\mathrm{or}}).
\label{eq:tv-triangle}
\end{align}
The first and last terms are $o_{\P}(1)$ by
Eqs.~\eqref{eq:BvM-fe} and \eqref{eq:BvM-or-assmp}.
For the middle term, Pinsker's inequality and the KL formula
for Gaussian distributions with common covariance give
\begin{align}
\mathrm{TV}(G_{\mathrm{fe}},G_{\mathrm{or}})
&\le\frac12\sqrt{n\omega\,
(\hat\theta_{\mathrm{fe}}-\hat\theta_{\mathrm{or}})^\top
V_0(\hat\theta_{\mathrm{fe}}-\hat\theta_{\mathrm{or}})}\nonumber\\
&\le\frac12\sqrt{n\omega\,\lambda_{\max}(V_0)}\,
\|\hat\theta_{\mathrm{fe}}-\hat\theta_{\mathrm{or}}\|.
\label{eq:tv-gauss-bound}
\end{align}
Therefore, by Eq.~\eqref{eq:center-shift},
\begin{equation}\label{eq:tv-gauss-rate}
\mathrm{TV}(G_{\mathrm{fe}},G_{\mathrm{or}})
=O_{\P}(\sqrt n\,r_n^2+r_n).
\end{equation}
Combining with Eq.~\eqref{eq:tv-triangle} and using $r_n\to0$ gives
\begin{equation}
\mathrm{TV}(q_{n,\mathrm{fe}},q_{n,\mathrm{or}})
=O_{\P}(\sqrt n\,r_n^2)+o_{\P}(1),
\end{equation}
which is Eq.~\eqref{eq:tv-bound-finite}. If $r_n=o(n^{-1/4})$,
then $\sqrt n\,r_n^2\to0$, and the feasible posterior inherits
the oracle BvM limit.
\end{proof}

\subsection{Finite dimensional projections of an infinite dimensional \texorpdfstring{$\theta$}{theta}}
\begin{theorem}[Posterior stability under orthogonal losses for infinite-dimensional $\theta$ (finite-dimensional projections)]
\label{thm:posterior-stability-proj}
Let $\Theta$ be a separable Hilbert space with norm $\|\cdot\|_\Theta$.
Let $\eta\in\mathcal H$ be a nuisance parameter with truth $\eta_0$, and define
\begin{equation}
\L_n(\theta;\eta):=\frac1n\sum_{i=1}^n \ell(O_i;\theta,\eta),
\qquad
L(\theta;\eta):=\E_{\P}\{\ell(O;\theta,\eta)\},
\end{equation}
where $\P$ is the true observational distribution. Define the oracle and feasible generalized posteriors
\begin{equation}
q_{n,\mathrm{or}}(\theta\mid \mathcal D_n)\ \propto\ \pi(\theta)\exp\{-\omega n \L_n(\theta;\eta_0)\},
\qquad
q_{n,\mathrm{fe}}(\theta\mid \mathcal D_n)\ \propto\ \pi(\theta)\exp\{-\omega n \L_n(\theta;\hat\eta)\}.
\end{equation}
Let $\hat\theta_{\mathrm{or}}:=\arg\min_{\theta\in\Theta}\L_n(\theta;\eta_0)$ and
$\hat\theta_{\mathrm{fe}}:=\arg\min_{\theta\in\Theta}\L_n(\theta;\hat\eta)$.

Fix $m\in\mathbb N$ and bounded linear functionals $T_1,\dots,T_m\in \Theta^\ast$, and write
$T(\theta):=(T_1(\theta),\dots,T_m(\theta))^\top\in\mathbb R^m$.
Let $q_{n,\mathrm{or}}^T$ and $q_{n,\mathrm{fe}}^T$ denote the induced (pushforward) posteriors on $\mathbb R^m$:
\begin{equation}
q_{n,\mathrm{or}}^T := q_{n,\mathrm{or}}\circ T^{-1},\qquad
q_{n,\mathrm{fe}}^T := q_{n,\mathrm{fe}}\circ T^{-1}.
\end{equation}

Assume:
\begin{enumerate}
\item[(B1)] \textbf{(Projection BvM, oracle, and feasible).}
There exists a bounded, self-adjoint, strictly positive operator
$V_0 := \nabla_{\theta\theta} L(\theta^\star;\eta_0)$ on $\Theta$ such that
\begin{align}
\mathrm{TV}\!\Bigl(q_{n,\mathrm{or}}^T(\cdot\mid\mathcal D_n),\
\mathcal N_m\!\bigl(T(\hat\theta_{\mathrm{or}}),\ (n\omega\, T V_0T^\top)^{-1}\bigr)\Bigr)
&= o_{\P}(1),\label{eq:BvM-or-proj}\\
\mathrm{TV}\!\Bigl(q_{n,\mathrm{fe}}^T(\cdot\mid\mathcal D_n),\
\mathcal N_m\!\bigl(T(\hat\theta_{\mathrm{fe}}),\ (n\omega\, T V_0T^\top)^{-1}\bigr)\Bigr)
&= o_{\P}(1),\label{eq:BvM-fe-proj}
\end{align}
where $T V_0T^\top$ is the $m\times m$ matrix with entries
$(T V_0T^\top)_{jk} := \langle T_j^\sharp,\ V_0 T_k^\sharp\rangle_\Theta$,
and $T_j^\sharp\in\Theta$ is the Riesz representer of $T_j$.

\item[(B2)] \textbf{(Orthogonality).} The loss is Neyman orthogonal at $(\theta^\star,\eta_0)$ in the sense of
Eq.~\eqref{def:orth-loss}.

\item[(B3)] \textbf{(Cross-fitting and projected center stability).}
$\hat\eta$ is obtained by sample splitting/cross-fitting with
$\|\hat\eta-\eta_0\|=:r_n\to0$. Assume additionally that the orthogonal M-estimation
regularity conditions yield the projected center-shift bound
$\|T(\hat\theta_{\mathrm{fe}})-T(\hat\theta_{\mathrm{or}})\|_2=O_{\P}(r_n^2)$.
\end{enumerate}

{Then}
\begin{equation}\label{eq:tv-proj-bound}
\mathrm{TV}\!\left(q_{n,\mathrm{fe}}^T(\cdot \mid \mathcal{D}_n),\ q_{n,\mathrm{or}}^T(\cdot \mid \mathcal{D}_n)\right)
\;=\; O_{\P}\!\bigl(\sqrt{n}\, r_n^{2}\bigr) \;+\; o_{\P}(1).
\end{equation}
In particular, if $r_n=o(n^{-1/4})$, then
$\mathrm{TV}(q_{n,\mathrm{fe}}^T,q_{n,\mathrm{or}}^T)\to 0$ in probability, so the feasible posterior inherits the oracle
projection BvM limit from Eq.~\eqref{eq:BvM-or-proj} for $T(\theta)$.
\end{theorem}

\begin{proof}
We omit the conditioning on $\mathcal D_n$ for readability.

\paragraph{Step 1 (Gaussian approximations in $\mathbb R^m$).}
By Eq.~\eqref{eq:BvM-or-proj}--\eqref{eq:BvM-fe-proj}, define the Gaussian approximations
\begin{equation}
G_{\mathrm{or}}^T := \mathcal N_m\!\bigl(T(\hat\theta_{\mathrm{or}}),\ (n\omega\, T V_0T^\top)^{-1}\bigr),
\qquad
G_{\mathrm{fe}}^T := \mathcal N_m\!\bigl(T(\hat\theta_{\mathrm{fe}}),\ (n\omega\, T V_0T^\top)^{-1}\bigr),
\end{equation}
so that
$\mathrm{TV}(q_{n,\mathrm{or}}^T,G_{\mathrm{or}}^T)=o_{\P}(1)$ and
$\mathrm{TV}(q_{n,\mathrm{fe}}^T,G_{\mathrm{fe}}^T)=o_{\P}(1)$.

\paragraph{Step 2 (Orthogonality implies a second-order center shift).}
Assumption (B3) supplies the projected center-shift bound
\begin{equation}\label{eq:proj-center-shift}
\|T(\hat\theta_{\mathrm{fe}})-T(\hat\theta_{\mathrm{or}})\|_2
\;=\;O_{\P}(r_n^2).
\end{equation}
This is the projected counterpart of the finite-dimensional M-estimation stability
argument; it is assumed here rather than derived from the Hilbert-space conditions.

\paragraph{Step 3 (Triangle inequality).}
\begin{align}
\mathrm{TV}(q_{n,\mathrm{fe}}^T,q_{n,\mathrm{or}}^T)
&\le
\mathrm{TV}(q_{n,\mathrm{fe}}^T,G_{\mathrm{fe}}^T)
+\mathrm{TV}(G_{\mathrm{fe}}^T,G_{\mathrm{or}}^T)
+\mathrm{TV}(G_{\mathrm{or}}^T,q_{n,\mathrm{or}}^T).
\label{eq:tv-triangle-proj}
\end{align}
The first and third terms are $o_{\P}(1)$. It remains to bound $\mathrm{TV}(G_{\mathrm{fe}}^T,G_{\mathrm{or}}^T)$.

\paragraph{Step 4 (TV bound between Gaussians with common covariance).}
Both Gaussians have the same covariance
$\Sigma_T=(n\omega\,T V_0T^\top)^{-1}$ and means
$\mu_{\mathrm{fe}}=T(\hat\theta_{\mathrm{fe}})$, $\mu_{\mathrm{or}}=T(\hat\theta_{\mathrm{or}})$.
For multivariate normals with equal covariance,
\begin{equation}
\mathrm{KL}(G_{\mathrm{fe}}^T\|G_{\mathrm{or}}^T)
=\frac12(\mu_{\mathrm{fe}}-\mu_{\mathrm{or}})^\top \Sigma_T^{-1}(\mu_{\mathrm{fe}}-\mu_{\mathrm{or}})
=\frac{n\omega}{2}(\mu_{\mathrm{fe}}-\mu_{\mathrm{or}})^\top (T V_0T^\top)(\mu_{\mathrm{fe}}-\mu_{\mathrm{or}}).
\end{equation}
By Pinsker's inequality,
\begin{align}
\mathrm{TV}(G_{\mathrm{fe}}^T,G_{\mathrm{or}}^T)
&\le \sqrt{\frac12\,\mathrm{KL}(G_{\mathrm{fe}}^T\|G_{\mathrm{or}}^T)}
\;\le\; \frac12\sqrt{n\omega\,\lambda_{\max}(T V_0T^\top)}\;\|\mu_{\mathrm{fe}}-\mu_{\mathrm{or}}\|_2.
\label{eq:tv-gauss-proj}
\end{align}
Using Eq.~\eqref{eq:proj-center-shift} and that $\lambda_{\max}(T V_0T^\top)=O(1)$ for fixed $T$,
we get
\begin{equation}
\mathrm{TV}(G_{\mathrm{fe}}^T,G_{\mathrm{or}}^T)=O_{\P}(\sqrt{n}\,r_n^2).
\end{equation}

\paragraph{Step 5 (Combine).}
Plugging into Eq.~\eqref{eq:tv-triangle-proj} yields
$\mathrm{TV}(q_{n,\mathrm{fe}}^T,q_{n,\mathrm{or}}^T)=O_{\P}(\sqrt{n}\,r_n^2)+o_{\P}(1)$.
If $r_n=o(n^{-1/4})$, then $\sqrt{n}\,r_n^2\to0$ and, hence,
$\mathrm{TV}(q_{n,\mathrm{fe}}^T,q_{n,\mathrm{or}}^T)\to0$ in probability.
\end{proof}

\begin{remark}[CATE as a special case]
If $\theta=\tau$ is the CATE function and $T$ is the evaluation map
$T(\tau)=(\tau(x_1),\dots,\tau(x_m))$, the theorem gives stability (and inherited BvM)
for any fixed finite set of evaluation points $x_{1:m}$, provided point evaluation is continuous
in the chosen function space $\Theta$.
\end{remark}


\section{Additional Details}
\label{appx:additional-details}

\subsection{Interpretation of generalized posteriors}
\label{appx:interpretation}
Our framework constructs a generalized (Gibbs) posterior over the causal estimand rather than a conventional likelihood-based Bayesian posterior. The distinction is that, while the generalized posterior is a normalized probability distribution over the estimand, the exponentiated loss replacing the likelihood need not constitute a conditional probability distribution over the observed data. Consequently, the generalized posterior need not arise by conditioning a joint probabilistic model of the data and estimand with the specified prior as its marginal. This does not imply that generalized Bayesian updating is incoherent: for additive losses and a fixed learning rate, sequential updating yields the same generalized posterior as updating on the combined data. Moreover, specific losses may admit a working-likelihood interpretation, as illustrated by our Gaussian pseudo-outcome example in Section~\ref{sec:theoretical_results}. Such a representation, however, does not amount to modeling the full observational data-generating process. Hence, the resulting uncertainty is interpreted as generalized-posterior uncertainty, with frequentist coverage requiring appropriate calibration and regularity conditions. We refer to \citep{JeremiasKnoblauch.2022} for a principled exposition on posterior generalizations.

\subsection{Nuisance estimation}
\label{appx:nuisance-estimation}

 The generalized posterior is constructed using a loss that depends on plug-in nuisance estimates $\hat\eta^S$, without restricting how these estimates are obtained. In particular, a Bayesian nuisance model can be fitted on each training fold and its posterior mean used as the corresponding plug-in estimator. This includes Bayesian nonparametric models, such as Gaussian processes. Our theoretical guarantees depend on the convergence rate of $\hat\eta^S$, rather than on the estimation procedure itself. Thus, Bayesian nuisance estimators are directly admissible provided their plug-in estimates satisfy the nuisance-rate and regularity conditions of Theorem~\ref{thm:posterior-stability-finite}.

\textbf{Why not integrate over the nuisance posterior?}
An alternative to plug-in estimation is to explicitly propagate nuisance uncertainty by marginalizing over a posterior distribution $\Pi(\eta^S\mid D_n)$, yielding
\[
\widetilde q_n^S(\theta\mid D_n)
=
\int q_n^S(\theta\mid D_n,\eta^S)\,
\Pi(d\eta^S\mid D_n).
\]
Our choice to employ plug-in nuisances instead has two motivations. First, our construction is modular: information from nuisance estimation enters the estimand-level update, but the resulting generalized posterior does not feed back into nuisance estimation. This one-way information flow is closely related to the motivation for cut-posterior approaches \citep{jacobBetterTogetherStatistical2017a,frazierCuttingFeedbackModularized2025}, which seek to avoid potentially harmful feedback between model components.

Second, Neyman-orthogonality makes explicit propagation of nuisance uncertainty unnecessary for the first-order frequentist guarantees established in our framework. Under the conditions of Theorem~\ref{thm:posterior-stability-finite}, orthogonality removes the first-order sensitivity of the estimating equation to nuisance perturbations, such that the resulting posterior displacement is of order $O_P(r_n^2)$. When $r_n=o(n^{-1/4})$, this displacement is negligible relative to the $n^{-1/2}$ posterior scale. This parallels orthogonal statistical learning, where the asymptotic variance is determined by the oracle influence function without an additional first-order contribution from nuisance estimation \citep{Chernozhukov.2018}. Consequently, explicitly integrating over the nuisance posterior is not required for the asymptotic uncertainty guarantees considered here, while doing so would require specifying and computing a potentially high-dimensional nuisance posterior. Cross-fitting ensures that the held-out observations are conditionally i.i.d. within each fold given the corresponding nuisance estimates, providing the independence structure used in our theoretical analysis of posterior stability under Neyman-orthogonality.

\subsection{Prior elicitation}
\label{appx:prior-elicitation}

A practical advantage of our framework is that prior beliefs can be specified directly on the causal estimand $\theta$, rather than indirectly through a model of the observational distribution and its nuisance components. For scalar estimands such as the ATE, prior information may come from historical studies or meta-analyses \citep{schmidliRobustMetaanalyticpredictivePriors2014}, or from domain experts whose beliefs about plausible treatment effects can be translated into prior distributions using established elicitation procedures \citep{ohaganUncertainJudgementsEliciting2006}. When substantive prior information is limited, weakly informative priors offer a practical alternative, with their scale chosen to reflect plausible effect sizes \citep{gelmanWeaklyInformativeDefault2008}. Importantly, we do not advocate noninformative priors as a general recommendation: prior specification should reflect the scientific context, available evidence, and scale of the estimand. In particular, when prior information is uncertain, sensitivity to reasonable alternative prior specifications should be assessed.

\subsection{Choice of cross-fitting folds}
\label{appx:cross-fitting-folds}

The number of cross-fitting folds $K$ governs a trade-off between nuisance estimation quality and computational cost. With $K$ folds, each nuisance estimator is trained on a fraction $(K-1)/K$ of the data, so increasing $K$ provides larger training samples but requires more nuisance model fits. Our theoretical guarantees assume a fixed $K$ as $n\to\infty$, but do not prescribe an optimal choice. In practice, common choices are $K\in\{2,5,10\}$ \citep{Chernozhukov.2018}, with $K=5$ providing a reasonable default. Smaller values may be preferable for computationally expensive nuisance learners, while larger values can be beneficial when nuisance estimation is sensitive to training sample size. Provided the required nuisance convergence rates and regularity conditions hold, the asymptotic guarantees remain valid for any fixed $K$. In finite samples, sensitivity to the choice of $K$ can be assessed empirically (see Appendix~\ref{appx:cross-fitting-choice}).

\newpage
\section{Extended Related Work}
\label{appx:extended-related-work}

\textbf{Generalized Bayesian inference:} Related generalization ideas appear in econometrics under the terms of quasi-/Laplace posterior \citep{chernozhukovMCMCApproachClassical2003}, and more recent optimization-centric formulations \citep{JeremiasKnoblauch.2022} give even broader formalizations of the same concept.
\citet{millerAsymptoticNormalityConcentration2021} provide general sufficient conditions for concentration, asymptotic normality (Bernstein-von Mises), and frequentist coverage of generalized posteriors, while \citet{syringCalibratingGeneralPosterior2019} develop a calibration procedure to achieve nominal frequentist coverage of the generalized posterior' credible sets.

\textbf{Likelihood-based Bayesian causal inference:} A large body of Bayesian causal inference proceeds by specifying a probabilistic model for (parts of) the observed data distribution and performing posterior inference under this model. Many approaches leverage flexible Bayesian nonparametric priors over high-dimensional nuisance components \citep[e.g.,][]{Antonelli.5132018, Ray.2020, Munoz.2011}; see \citet{Linero.1162021} for an overview.

\textbf{Comparison to \citet{jacobBetterTogetherStatistical2017a} and \citet{frazierCuttingFeedbackModularized2025}:} While \citet{jacobBetterTogetherStatistical2017a} and, in the generalized Bayes setting, \citet{frazierCuttingFeedbackModularized2025} study modular inference for fully specified model modules, our cut structure is realized through an identification-driven Neyman-orthogonal loss, which is what delivers robustness to nuisance estimation error (Theorem~\ref{thm:posterior-stability-finite}).

\textbf{Comparison to \citet{yiuSemiparametricPosteriorCorrections2025}:} The approach of \citet{yiuSemiparametricPosteriorCorrections2025} starts from a conventional nonparametric Bayesian posterior over the full data-generating distribution and post-processes the marginal posterior of each target functional with a Bayesian-bootstrap one-step correction, yielding a corrected posterior that satisfies a Bernstein--von Mises (BvM) theorem with semiparametric-efficient width; the causal example developed there is the ATT. The difference to our work is along four axes:

\begin{enumerate}
    \item \textit{Where the prior lives.}
    Yiu et al. require a prior and posterior on $P_{\mathrm{obs}}$; the marginal on $\theta$ remains a pushforward $\bar{Q}_{\#}\pi$, so there is no direct control over the induced prior on the estimand: precisely the difficulty we describe as point (4) of Section~\ref{sec:standard-bayes-hard}. Our framework never forms a posterior over $P_{\mathrm{obs}}$ and places the prior on $\theta$ directly, thus resolving this difficulty by construction.

    \item \textit{Where the debiasing enters.}
    They correct posterior draws post hoc using the efficient influence function. We build robustness into the update itself, through the identification-driven orthogonal loss. Consequently, any loss from the causal ML literature (e.g., DR-learner, R-learner) can be used as-is, without deriving and implementing an estimand-specific correction step.

    \item \textit{Scope of estimands.}
    Their construction targets low-dimensional functionals admitting an efficient influence function. Ours applies to any identified estimand expressible as an M-estimation problem, including infinite-dimensional targets such as the CATE function, where our guarantee is stated for finite-dimensional projections (Theorem~\ref{thm:posterior-stability-proj}) and where a one-step correction of a marginal posterior is not the natural object.

    \item \textit{What is required in practice.}
    Their method attaches to a posterior sampling algorithm for the full nonparametric model. Ours requires only a fitted nuisance learner and a loss, which is easier as the components are already present in a standard causal ML pipeline.
\end{enumerate}

\textbf{Concurrent work on Bayesian semiparametric inference:}
In concurrent work, \citet{sabbaghSemiparametricBayesianInference2026} study Bayesian semiparametric inference for target parameters defined through estimating equations involving estimated nuisance functions. Using a nonparametric Bayesian formulation based on the Dirichlet process and Bayesian bootstrap, they investigate when two-step plug-in inference remains valid despite not explicitly propagating posterior uncertainty about the nuisance functions. They show how Neyman-orthogonality justifies this approach and, interestingly, identify settings in which the posterior is asymptotically unaffected by nuisance estimation even without orthogonality, provided the nuisance estimator is consistent. In a follow-up paper, \citet{sabbaghBootstrapValidityBayesian2026} extend the analysis to sample splitting and cross-fitting, relaxing Donsker-type restrictions and establishing asymptotic properties of Bayesian bootstrap posteriors with flexible nuisance estimators. Together, these works provide a complementary perspective on Bayesian inference with estimated nuisances, highlighting the interplay between orthogonal estimating equations, posterior uncertainty, and the Bayesian bootstrap.

\newpage
\section{Additional Experiments}
\label{app:other-experiments}

\subsection{CATE estimation}

\begin{figure}[!ht]
\centering
\includegraphics[width = 1\columnwidth]{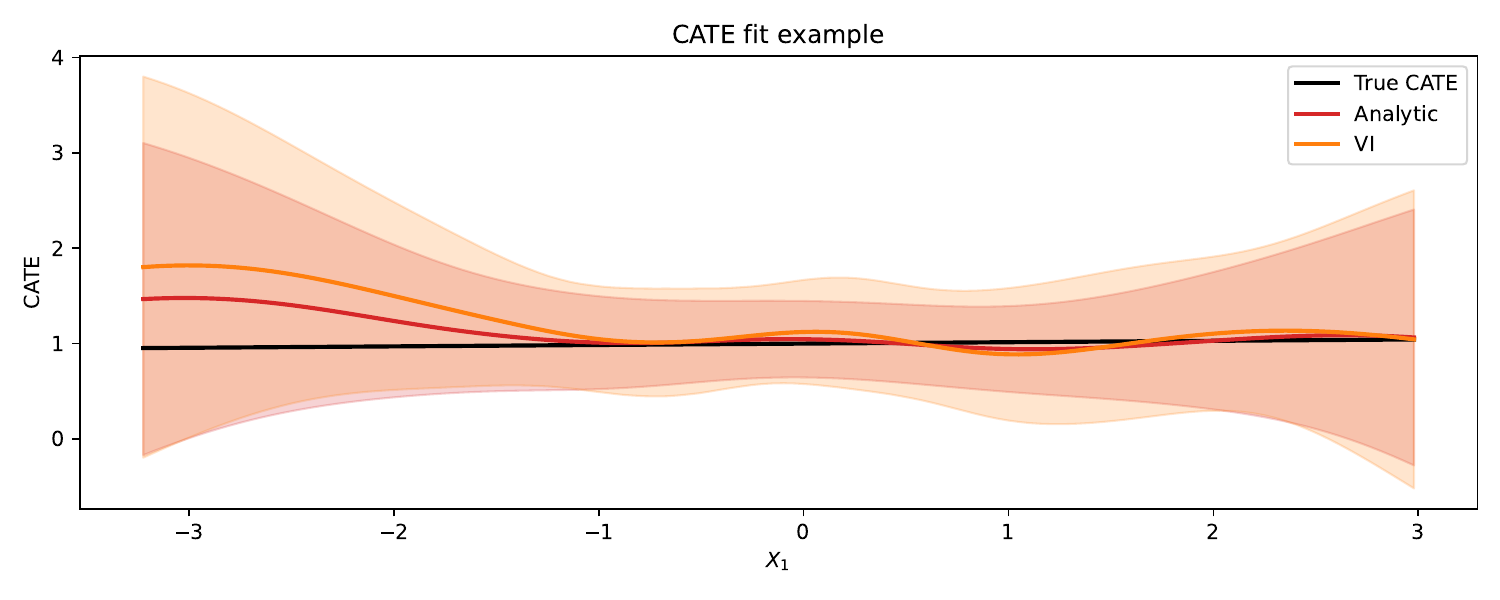}
\caption{
\scriptsize{
\textbf{(back-door CATE)} An example of a Gaussian Process fit (mean and $95\%$ CrI) of the generalized posterior of CATE, with the $\mathrm{DR}$ loss, in $\mathcal{D}_4$, at $n=1000$. The solution fitted with variational inference (VI) uses an Inducing Point GP. For implementation details, see Appendix~\ref{app:implementation}.
}
}
\label{fig:cate-example}
\end{figure}

In Figure~\ref{fig:cate-example} we show an example of a fit of the generalized posterior for CATE. We report the coverage and length of credible interval for the back-door CATE setting in Tables~\ref{table:cate-coverage-analytic} and~\ref{table:cate-length-analytic}. Note that the uncertainty around the \textit{function} CATE is not a single interval, but rather a band around the mean fit. We report coverage and length metrics by averaging the point-wise uncertainty at $K=100$ randomly sampled covariates $X$ sampled from the underlying data-generating process.

Note that the CATE results are heavily influenced by the choice of kernel, which we have not optimized (for our implementation details, see Appendix~\ref{app:implementation}). Our aim here is to demonstrate the applicability of our approach to different settings and to show it can achieve frequentist-calibrated uncertainty coverage.

\begin{table}[!h]
\centering
\resizebox{\textwidth}{!}{%
\begin{tabular}{lcccccccccc}
\hline
Strategy & Orthogonal & $\mathcal{D}_{1}$ & $\mathcal{D}_{2}$ & $\mathcal{D}_{3}$ & $\mathcal{D}_{4}$ & $\mathcal{D}_{5}$ & $\mathcal{D}_{6}$ & $\mathcal{D}_{7}$ & $\mathcal{D}_{8}$ & $\mathcal{D}_{9}$ \\ 
\hline
RA & \xmark & $\textbf{1.000}$ \scriptsize (0.964-1.000) & $1.000$ \scriptsize (0.964-1.000) & $\textbf{1.000}$ \scriptsize (0.964-1.000) & $\textbf{1.000}$ \scriptsize (0.964-1.000) & $\textbf{1.000}$ \scriptsize (0.964-1.000) & $0.500$ \scriptsize (0.398-0.602) & $0.760$ \scriptsize (0.664-0.840) & $0.420$ \scriptsize (0.322-0.523) & $0.000$ \scriptsize (0.000-0.036) \\ 
IPW & \xmark & $\textbf{1.000}$ \scriptsize (0.964-1.000) & $0.900$ \scriptsize (0.824-0.951) & $\textbf{1.000}$ \scriptsize (0.964-1.000) & $0.820$ \scriptsize (0.731-0.890) & $0.890$ \scriptsize (0.812-0.944) & $\textbf{1.000}$ \scriptsize (0.964-1.000) & $\textbf{0.790}$ \scriptsize (0.697-0.865) & $\textbf{1.000}$ \scriptsize (0.964-1.000) & $\textbf{0.860}$ \scriptsize (0.776-0.921) \\ 
DR & \cmark & $\textbf{1.000}$ \scriptsize (0.964-1.000) & $\textbf{0.990}$ \scriptsize (0.946-1.000) & $\textbf{1.000}$ \scriptsize (0.964-1.000) & $\textbf{1.000}$ \scriptsize (0.964-1.000) & $\textbf{1.000}$ \scriptsize (0.964-1.000) & $0.800$ \scriptsize (0.708-0.873) & $0.770$ \scriptsize (0.675-0.848) & $\textbf{1.000}$ \scriptsize (0.964-1.000) & $0.540$ \scriptsize (0.437-0.640) \\ 
\hline
\end{tabular}
 
}

\caption{
\scriptsize{
\textbf{(back-door CATE)} The $95\%$ $\mathrm{CrI}^{(r)}_{0.95}$ credible interval's coverage $\hat C$ (and its $95\%$ confidence interval) for each strategy $S \in \{\mathrm{DR,IPW,RA}\}$ across $9$ synthetic datasets $\{\mathcal{D}_1,\ldots,\mathcal{D}_9\}$. Results are computed from $R=50$ repetitions at sample size $n=1000$. Closest to exact frequentist calibration at $95\%$ is in \textbf{bold} We report coverage metrics by averaging the point-wise uncertainty at $K=100$ randomly sampled covariates $X$ sampled from the underlying data-generating process. .
}
}
\label{table:cate-coverage-analytic}
\end{table}

\begin{table}[!h]
\centering
\resizebox{\textwidth}{!}{%
\begin{tabular}{lcccccccccc}
\hline
Strategy & Orthogonal & $\mathcal{D}_{1}$ & $\mathcal{D}_{2}$ & $\mathcal{D}_{3}$ & $\mathcal{D}_{4}$ & $\mathcal{D}_{5}$ & $\mathcal{D}_{6}$ & $\mathcal{D}_{7}$ & $\mathcal{D}_{8}$ & $\mathcal{D}_{9}$ \\ 
\hline
RA & \xmark & $\textbf{1.292 (0.000)}$ & $\textbf{0.939 (0.000)}$ & $\textbf{0.801 (0.000)}$ & $1.606 (0.000)$ & $1.490 (0.000)$ & $\textcolor{gray}{\cancel{1.505 (0.000)}}$ & $\textcolor{gray}{\cancel{0.985 (0.000)}}$ & $\textcolor{gray}{\cancel{3.563 (0.000)}}$ & $\textcolor{gray}{\cancel{0.436 (0.000)}}$ \\ 
IPW & \xmark & $2.122 (0.000)$ & $1.965 (0.000)$ & $2.318 (0.000)$ & $\textcolor{gray}{\cancel{3.718 (0.000)}}$ & $\textcolor{gray}{\cancel{1.971 (0.000)}}$ & $\textbf{3.417 (0.000)}$ & $\textcolor{gray}{\cancel{2.361 (0.000)}}$ & $\textbf{5.484 (0.000)}$ & $\textcolor{gray}{\cancel{3.016 (0.000)}}$ \\ 
DR & \cmark & $1.378 (0.000)$ & $1.153 (0.000)$ & $1.524 (0.000)$ & $\textbf{1.554 (0.000)}$ & $\textbf{1.321 (0.000)}$ & $\textcolor{gray}{\cancel{1.435 (0.000)}}$ & $\textcolor{gray}{\cancel{2.308 (0.000)}}$ & $5.556 (0.000)$ & $\textcolor{gray}{\cancel{0.932 (0.000)}}$ \\ 
\hline
\end{tabular}
 
}
\caption{
\scriptsize{
\textbf{(back-door CATE)} The average length (and standard deviation) of the $95\%$ $\mathrm{CrI}^{(r)}_{0.95}$ credible interval across $R=50$ repetitions for each strategy $S \in \{\mathrm{DR,IPW,RA}\}$ across $9$ synthetic datasets $\{\mathcal{D}_1,\ldots,\mathcal{D}_9\}$. Results are computed at sample size $n=1000$. The lengths of unfaithful credible intervals are \textcolor{gray}{striked out in gray}. The narrowest $\mathrm{CrI}^{(r)}_{0.95}$ is length is highlighted in \textbf{bold}. We report length metrics by averaging the point-wise uncertainty at $K=100$ randomly sampled covariates $X$ sampled from the underlying data-generating process. 
}
}
\label{table:cate-length-analytic}
\end{table}

\subsection{Credible interval lengths for ATE estimation}
\label{appx:ate-length-exp}

\begin{table*}[h]
\vspace{-0.1cm}
\centering
\resizebox{\textwidth}{!}{%
\begin{tabular}{lcccccccccc}
\hline
Strategy & Orthogonal & $\mathcal{D}_{1}$ & $\mathcal{D}_{2}$ & $\mathcal{D}_{3}$ & $\mathcal{D}_{4}$ & $\mathcal{D}_{5}$ & $\mathcal{D}_{6}$ & $\mathcal{D}_{7}$ & $\mathcal{D}_{8}$ & $\mathcal{D}_{9}$ \\
\hline
RA & \xmark & $\textcolor{gray}{\cancel{0.099 (0.002)}}$ & $\textcolor{gray}{\cancel{0.112 (0.004)}}$ & $\textcolor{gray}{\cancel{0.113 (0.004)}}$ & $\textcolor{gray}{\cancel{0.139 (0.007)}}$ & $\textcolor{gray}{\cancel{0.099 (0.002)}}$ & $\textcolor{gray}{\cancel{0.112 (0.009)}}$ & $\textcolor{gray}{\cancel{0.108 (0.009)}}$ & $\textcolor{gray}{\cancel{0.108 (0.002)}}$ & $\textcolor{gray}{\cancel{0.156 (0.012)}}$ \\
IPW & \xmark & $0.479 (0.018)$ & $0.490 (0.019)$ & $0.542 (0.024)$ & $0.687 (0.067)$ & $0.649 (0.220)$ & $\textcolor{gray}{\cancel{0.693 (0.256)}}$ & $0.730 (0.246)$ & $0.644 (0.222)$ & $2.835 (0.055)$ \\
AIPW & \cmark & $\textbf{0.263 (0.010)}$ & $\textbf{0.273 (0.009)}$ & $\textbf{0.272 (0.008)}$ & $\textbf{0.289 (0.009)}$ & $\textbf{0.275 (0.021)}$ & $\textcolor{gray}{\cancel{0.446 (0.136)}}$ & $\textbf{0.587 (0.163)}$ & $\textbf{0.336 (0.058)}$ & $\textbf{0.466 (0.012)}$ \\
\hline
\end{tabular}
 
}
\caption{
\scriptsize{
\textbf{CrI length for back-door ATE.} The average length (and standard deviation) of the $95\%$ $\mathrm{CrI}^{(r)}_{0.95}$ credible interval across $R=50$ repetitions for each strategy $S \in \{\mathrm{AIPW,IPW,RA}\}$ across $9$ synthetic datasets $\{\mathcal{D}_1,\ldots,\mathcal{D}_9\}$. Results are computed at sample size $n=1000$. The lengths of unfaithful credible intervals are \textcolor{gray}{striked out in gray}. The narrowest $\mathrm{CrI}^{(r)}_{0.95}$ is length is highlighted in \textbf{bold}.
}
}
\label{table:ate-length-analytic}
\end{table*}

Table~\ref{table:ate-length-analytic} reports the credible interval lengths corresponding to the ATE coverage results in Table~\ref{table:ate-coverage-analytic}. Among methods achieving faithful coverage, the AIPW-based generalized posterior generally produces the shortest intervals, consistent with its theoretical efficiency properties.

\subsection{Comparison of calibration methods}
\label{app:e10-calibration-comparison}

To complement the sensitivity analysis in the main text, we compare the uncalibrated choice $\omega=1$ with analytic plug-in and bootstrap calibration across the nine synthetic settings, at $n\in\{200,1000\}$. Tables~\ref{table:e10_regime_compact} and~\ref{table:e10_coverage_n1000} show that both calibration methods achieve comparable, generally faithful coverage, with the analytic plug-in being computationally cheaper. Both substantially improve upon the uncalibrated posterior, although neither resolves persistent undercoverage in the limited-overlap setting $\mathcal{D}_6$.

\begin{table}[!h]
\centering
\resizebox{\textwidth}{!}{%
\begin{tabular}{llcccc}
\hline
$n$ & Calibration & Mean cov.\ (excl.\ $\mathcal{D}_6$) & Range (excl.\ $\mathcal{D}_6$) & Cov.\ $\mathcal{D}_6$ & Time [ms] \\
\hline
$200$ & $\omega=1$ (uncalibrated) & $0.377$ & $0.048$--$0.518$ & $0.174$ & -- \\
$200$ & Analytic plug-in $\hat\omega$ & $0.963$ & $0.954$--$0.972$ & $0.620$ & $0.08$ \\
$200$ & Bootstrap $\hat\omega$ & $0.962$ & $0.944$--$0.972$ & $0.612$ & $0.97$ \\
\hline
$1000$ & $\omega=1$ (uncalibrated) & $0.494$ & $0.308$--$0.630$ & $0.128$ & -- \\
$1000$ & Analytic plug-in $\hat\omega$ & $0.960$ & $0.946$--$0.978$ & $0.572$ & $0.08$ \\
$1000$ & Bootstrap $\hat\omega$ & $0.955$ & $0.930$--$0.980$ & $0.564$ & $0.97$ \\
\hline
\end{tabular}

}
\caption{
\scriptsize{
\textbf{Calibration regimes.} Coverage $\hat C$ of the $95\%$ credible interval of the $\mathrm{AIPW}$ generalized posterior with an uncalibrated $\omega=1$, the plug-in $\hat\omega=\widehat{\Var}(\hat Y^{\mathrm{AIPW}})^{-1}$, and a bootstrap $\hat\omega=(n\,\widehat{\Var}_{\mathrm{boot}}(\hat\theta^{\mathrm{AIPW}}))^{-1}$ computed from $100$ nonparametric bootstrap resamples of the cross-fitted pseudo-outcomes. We report the mean and range over $\{\mathcal{D}_1,\ldots,\mathcal{D}_9\}\setminus\{\mathcal{D}_6\}$, the coverage on $\mathcal{D}_6$, and the median wall-clock time of the calibration step. Results are computed from $R=500$ repetitions per dataset under a diffuse prior $\mathcal{N}(0,100^2)$.
}
}
\label{table:e10_regime_compact}
\end{table}

\begin{table}[!h]
\centering
\resizebox{\textwidth}{!}{%
\begin{tabular}{lccccccccc}
\hline
Calibration & $\mathcal{D}_{1}$ & $\mathcal{D}_{2}$ & $\mathcal{D}_{3}$ & $\mathcal{D}_{4}$ & $\mathcal{D}_{5}$ & $\mathcal{D}_{6}$ & $\mathcal{D}_{7}$ & $\mathcal{D}_{8}$ & $\mathcal{D}_{9}$ \\
\hline
$\omega=1$ (uncalibrated) & $0.610$ \scriptsize (0.566-0.653) & $0.592$ \scriptsize (0.547-0.635) & $0.630$ \scriptsize (0.586-0.672) & $0.548$ \scriptsize (0.503-0.592) & $0.510$ \scriptsize (0.465-0.555) & $0.128$ \scriptsize (0.100-0.160) & $0.308$ \scriptsize (0.268-0.351) & $0.338$ \scriptsize (0.297-0.381) & $0.418$ \scriptsize (0.374-0.463) \\
Analytic plug-in $\hat\omega$ & $0.946$ \scriptsize (0.922-0.964) & $0.966$ \scriptsize (0.946-0.980) & $0.972$ \scriptsize (0.953-0.985) & $\textbf{0.950}$ \scriptsize (0.927-0.967) & $\textbf{0.950}$ \scriptsize (0.927-0.967) & $\textbf{0.572}$ \scriptsize (0.527-0.616) & $\textbf{0.952}$ \scriptsize (0.929-0.969) & $\textbf{0.970}$ \scriptsize (0.951-0.983) & $\textbf{0.978}$ \scriptsize (0.961-0.989) \\
Bootstrap $\hat\omega$ & $\textbf{0.952}$ \scriptsize (0.929-0.969) & $\textbf{0.956}$ \scriptsize (0.934-0.972) & $\textbf{0.968}$ \scriptsize (0.949-0.982) & $0.938$ \scriptsize (0.913-0.957) & $0.944$ \scriptsize (0.920-0.962) & $0.564$ \scriptsize (0.519-0.608) & $0.930$ \scriptsize (0.904-0.951) & $0.970$ \scriptsize (0.951-0.983) & $0.980$ \scriptsize (0.964-0.990) \\
\hline
\end{tabular}

}
\caption{
\scriptsize{
\textbf{Coverage by calibration regime.} The $95\%$ $\mathrm{CrI}^{(r)}_{0.95}$ credible interval's coverage $\hat C$ (and its $95\%$ confidence interval) of the $\mathrm{AIPW}$ generalized posterior with an uncalibrated $\omega=1$, the plug-in $\hat\omega$, and the bootstrap $\hat\omega$ (Table~\ref{table:e10_regime_compact}) across $9$ synthetic datasets $\{\mathcal{D}_1,\ldots,\mathcal{D}_9\}$. Results are computed from $R=500$ repetitions at sample size $n=1000$. Closest to exact frequentist calibration at $95\%$ is in \textbf{bold}.
}
}
\label{table:e10_coverage_n1000}
\end{table}

\subsection{Comparison with a nonparametric Bayesian baseline}

We compare the AIPW-based generalized posterior with a Bayesian Additive Regression Trees (BART) response-surface model, implemented using PyMC-BART, on $\mathcal{D}_1$ and $\mathcal{D}_9$. Both methods are evaluated on the same $R=50$ simulated datasets with $n=500$ observations.

\begin{table}[!h]
\centering
\resizebox{\textwidth}{!}{%
\begin{tabular}{llcccc}
\hline
DGP & Method & Coverage & CrI length & MSE & Time [s] \\
\hline
$\mathcal{D}_{1}$ & BART (response surface) & $0.980$ \scriptsize (0.894-0.999) & $0.745$ & $0.0321$ & $33.54$ \\
 & Generalized Bayes (AIPW, ours) & $0.980$ \scriptsize (0.894-0.999) & $0.499$ & $0.0128$ & $0.04$ \\
\hline
$\mathcal{D}_{9}$ & BART (response surface) & $1.000$ \scriptsize (0.929-1.000) & $2.195$ & $0.2949$ & $30.51$ \\
 & Generalized Bayes (AIPW, ours) & $0.960$ \scriptsize (0.863-0.995) & $0.777$ & $0.0275$ & $0.04$ \\
\hline
\end{tabular}

}
\caption{
\scriptsize{
\textbf{Comparison to BART.} The $95\%$ credible interval's coverage $\hat C$ (and its $95\%$ confidence interval), mean CrI length, MSE of the posterior mean, and mean runtime per repetition for an off-the-shelf BART response-surface model (PyMC-BART, $50$ trees, one chain) and the $\mathrm{AIPW}$ generalized posterior, fitted to the same datasets. Results are computed from $R=50$ paired repetitions at sample size $n=500$; at this $R$, coverage differences are not resolved.
}
}
\label{table:e2_bart_baseline}
\end{table}

Table~\ref{table:e2_bart_baseline} shows that, while both approaches achieve high empirical coverage, our generalized posterior provides substantially shorter credible intervals and lower estimation error at considerably lower computational cost. Given the limited number of repetitions, we do not draw conclusions from small differences in coverage.

\subsection{Sensitivity to cross-fitting}
\label{appx:cross-fitting-choice}

We investigate the effect of the number of cross-fitting folds $K\in\{2,5,10\}$ on the AIPW-based generalized posterior, considering $\mathcal{D}_6$ and $\mathcal{D}_8$ with $n=1000$ and $R=500$ repetitions. To isolate the effect of $K$, we use the same calibration parameter $\hat\omega$ across all configurations.

Table~\ref{table:e4_k_sweep} shows that $K=5$ and $K=10$ yield similar performance, whereas using only two folds can substantially degrade coverage. This supports $K=5$ as a reasonable practical default, although increasing the number of folds does not resolve the difficulties caused by limited overlap. Naturally, this recommendation is sensitive to sample size and choice of nuisance estimation.

\begin{table}[!h]
\centering
\small
\begin{tabular}{lcccc}
\hline
 & \multicolumn{2}{c}{$\mathcal{D}_{6}$} & \multicolumn{2}{c}{$\mathcal{D}_{8}$} \\
Folds & Coverage & MSE & Coverage & MSE \\
\hline
$K=2$ & $0.560$ \scriptsize (0.515-0.604) & $0.129$ & $0.778$ \scriptsize (0.739-0.814) & $0.072$ \\
$K=5$ & $0.540$ \scriptsize (0.495-0.584) & $0.130$ & $0.964$ \scriptsize (0.944-0.979) & $0.029$ \\
$K=10$ & $0.554$ \scriptsize (0.509-0.598) & $0.129$ & $0.972$ \scriptsize (0.953-0.985) & $0.025$ \\
\hline
\end{tabular}

\caption{
\scriptsize{
\textbf{Choice of cross-fitting folds $K$.} The $95\%$ credible interval's coverage $\hat C$ (and its $95\%$ confidence interval) and MSE of the posterior mean of the $\mathrm{AIPW}$ generalized posterior for $K\in\{2,5,10\}$ on $\mathcal{D}_6$ (limited overlap) and $\mathcal{D}_8$ (high-dimensional sparse confounding). Results are computed from $R=500$ repetitions at sample size $n=1000$. All arms share the same $\hat\omega$, so the mean CrI length does not change with $K$ ($0.607$ on $\mathcal{D}_6$, $0.668$ on $\mathcal{D}_8$).
}
}
\label{table:e4_k_sweep}
\end{table}

\subsection{Sensitivity to nuisance estimation}

We investigate the effect of different nuisance estimators on the AIPW-based generalized posterior, considering linear/ridge regression, tree ensembles, neural networks, and Bayesian Gaussian process posterior-mean plug-ins. Tables~\ref{table:e5_learners_main} and~\ref{table:e6_bayesian_plugin} show that faithful uncertainty quantification can be achieved with both frequentist and Bayesian nuisance estimators, but depends on their estimation accuracy, consistent with our theoretical rate conditions. A broader comparison across all nine settings is provided in Table~\ref{table:e5x_coverage_all_dgps}.

\begin{table}[!h]
\centering
\resizebox{\textwidth}{!}{%
\begin{tabular}{lccccccc}
\hline
 &  & \multicolumn{2}{c}{$\mathcal{D}_{4}$} & \multicolumn{2}{c}{$\mathcal{D}_{5}$} & \multicolumn{2}{c}{$\mathcal{D}_{9}$} \\
Nuisance learner & Bayesian & Coverage & MSE & Coverage & MSE & Coverage & MSE \\
\hline
Linear/ridge & \xmark & $\textbf{0.952}$ \scriptsize (0.929-0.969) & $0.006$ & $\textbf{0.934}$ \scriptsize (0.909-0.954) & $0.008$ & $\textbf{0.958}$ \scriptsize (0.937-0.974) & $0.013$ \\
Tree ensemble & \xmark & $0.892$ \scriptsize (0.861-0.918) & $0.010$ & $0.822$ \scriptsize (0.786-0.855) & $0.015$ & $0.980$ \scriptsize (0.964-0.990) & $0.011$ \\
Small MLP & \xmark & $0.882$ \scriptsize (0.850-0.909) & $0.010$ & $0.280$ \scriptsize (0.241-0.322) & $0.051$ & $0.900$ \scriptsize (0.870-0.925) & $0.023$ \\
Bayesian GP (RFF) post.\ mean & \cmark & $0.956$ \scriptsize (0.934-0.972) & $0.006$ & $0.804$ \scriptsize (0.766-0.838) & $0.017$ & $0.992$ \scriptsize (0.980-0.998) & $0.010$ \\
\hline
\end{tabular}

}
\caption{
\scriptsize{
\textbf{Flexible and Bayesian nuisance estimators.} The $95\%$ credible interval's coverage $\hat C$ (and its $95\%$ confidence interval) and MSE of the posterior mean of the $\mathrm{AIPW}$ generalized posterior with four nuisance learners, on $\mathcal{D}_4$, $\mathcal{D}_5$, and $\mathcal{D}_9$. The Bayesian learner plugs in the fold-wise posterior mean of a Gaussian process (random-feature approximation). $\hat\omega$ is shared across learners, so the CrI length is identical within a dataset. Results are computed from $R=500$ repetitions at sample size $n=1000$ with $K=5$ folds. Closest to exact frequentist calibration at $95\%$ is in \textbf{bold}.
}
}
\label{table:e5_learners_main}
\end{table}

\begin{table}[!h]
\centering
\small
\begin{tabular}{lcccc}
\hline
DGP & Cov.: GP post.\ mean & Cov.: linear/ridge & MSE: GP & MSE: ridge \\
\hline
$\mathcal{D}_{4}$ & $0.956$ \scriptsize (0.934-0.972) & $0.952$ \scriptsize (0.929-0.969) & $0.006$ & $0.006$ \\
$\mathcal{D}_{5}$ & $0.804$ \scriptsize (0.766-0.838) & $0.934$ \scriptsize (0.909-0.954) & $0.017$ & $0.008$ \\
$\mathcal{D}_{9}$ & $0.992$ \scriptsize (0.980-0.998) & $0.958$ \scriptsize (0.937-0.974) & $0.010$ & $0.013$ \\
\hline
\end{tabular}

\caption{
\scriptsize{
\textbf{Bayesian nuisance posterior mean as plug-in.} The $95\%$ credible interval's coverage $\hat C$ (and its $95\%$ confidence interval) and MSE of the posterior mean when the nuisances are the fold-wise posterior mean of a Gaussian process (random-feature approximation), compared with the linear/ridge baseline. Results are computed from $R=500$ repetitions at sample size $n=1000$ with $K=5$ folds and a shared $\hat\omega$. Coverage is near nominal on $\mathcal{D}_4$ and $\mathcal{D}_9$ but not on $\mathcal{D}_5$, consistent with the nuisance-rate condition.
}
}
\label{table:e6_bayesian_plugin}
\end{table}

\begin{table}[!h]
\centering
\resizebox{\textwidth}{!}{%
\begin{tabular}{lccccccccc}
\hline
Nuisance learner & $\mathcal{D}_{1}$ & $\mathcal{D}_{2}$ & $\mathcal{D}_{3}$ & $\mathcal{D}_{4}$ & $\mathcal{D}_{5}$ & $\mathcal{D}_{6}$ & $\mathcal{D}_{7}$ & $\mathcal{D}_{8}$ & $\mathcal{D}_{9}$ \\
\hline
Linear/ridge & $0.980$ \scriptsize (0.894-0.999) & $0.920$ \scriptsize (0.808-0.978) & $0.980$ \scriptsize (0.894-0.999) & $\textbf{0.940}$ \scriptsize (0.835-0.987) & $\textbf{0.940}$ \scriptsize (0.835-0.987) & $0.620$ \scriptsize (0.472-0.753) & $\textbf{0.960}$ \scriptsize (0.863-0.995) & $\textbf{1.000}$ \scriptsize (0.929-1.000) & $1.000$ \scriptsize (0.929-1.000) \\
Tree ensemble & $\textbf{0.920}$ \scriptsize (0.808-0.978) & $\textbf{0.940}$ \scriptsize (0.835-0.987) & $\textbf{0.940}$ \scriptsize (0.835-0.987) & $0.860$ \scriptsize (0.733-0.942) & $0.860$ \scriptsize (0.733-0.942) & $0.400$ \scriptsize (0.264-0.548) & $0.920$ \scriptsize (0.808-0.978) & $1.000$ \scriptsize (0.929-1.000) & $\textbf{0.960}$ \scriptsize (0.863-0.995) \\
Small MLP & $1.000$ \scriptsize (0.929-1.000) & $0.900$ \scriptsize (0.782-0.967) & $0.980$ \scriptsize (0.894-0.999) & $0.920$ \scriptsize (0.808-0.978) & $0.260$ \scriptsize (0.146-0.403) & $\textbf{0.740}$ \scriptsize (0.597-0.854) & $0.960$ \scriptsize (0.863-0.995) & $1.000$ \scriptsize (0.929-1.000) & $0.920$ \scriptsize (0.808-0.978) \\
Bayesian GP (RFF) post.\ mean & $0.980$ \scriptsize (0.894-0.999) & $0.940$ \scriptsize (0.835-0.987) & $1.000$ \scriptsize (0.929-1.000) & $0.920$ \scriptsize (0.808-0.978) & $0.800$ \scriptsize (0.663-0.900) & $0.360$ \scriptsize (0.229-0.508) & $0.940$ \scriptsize (0.835-0.987) & $0.860$ \scriptsize (0.733-0.942) & $1.000$ \scriptsize (0.929-1.000) \\
\hline
\end{tabular}

}
\caption{
\scriptsize{
\textbf{Nuisance learners across all datasets.} The $95\%$ $\mathrm{CrI}^{(r)}_{0.95}$ credible interval's coverage $\hat C$ (and its $95\%$ confidence interval) of the $\mathrm{AIPW}$ generalized posterior for each nuisance learner across $9$ synthetic datasets $\{\mathcal{D}_1,\ldots,\mathcal{D}_9\}$. Results are computed from $R=50$ repetitions at sample size $n=1000$ with $K=5$ folds and a shared $\hat\omega$. Closest to exact frequentist calibration at $95\%$ is in \textbf{bold}.
}
}
\label{table:e5x_coverage_all_dgps}
\end{table}

\subsection{Semi-synthetic IHDP benchmark}

We further evaluate ATE estimation on the semi-synthetic IHDP benchmark, considering both IHDP-100 and IHDP-1000 replications. Table~\ref{table:e8_ihdp_ate} presents the results for the RA, IPW, and AIPW-based generalized posteriors. Unlike in most of our synthetic settings, the AIPW-based posterior does not achieve nominal coverage on this benchmark, illustrating that Neyman orthogonality alone does not guarantee faithful uncertainty quantification in finite samples.

\begin{table}[!h]
\centering
\resizebox{\textwidth}{!}{%
\begin{tabular}{lccccccc}
\hline
 &  & \multicolumn{3}{c}{IHDP-100 ($R=100$)} & \multicolumn{3}{c}{IHDP-1000 ($R=1000$)} \\
Strategy & Orthogonal & Coverage & CrI length & MSE & Coverage & CrI length & MSE \\
\hline
RA & \xmark & $0.420$ \scriptsize (0.322-0.523) & $0.862$ & $2.098$ & $0.392$ \scriptsize (0.362-0.423) & $0.883$ & $1.808$ \\
IPW & \xmark & $0.990$ \scriptsize (0.946-1.000) & $6.715$ & $2.251$ & $0.988$ \scriptsize (0.979-0.994) & $6.860$ & $4.030$ \\
AIPW & \cmark & $0.730$ \scriptsize (0.632-0.814) & $1.286$ & $2.212$ & $0.688$ \scriptsize (0.658-0.717) & $1.307$ & $1.786$ \\
\hline
\end{tabular}

}
\caption{
\scriptsize{
\textbf{Semi-synthetic IHDP benchmark (ATE).} The $95\%$ credible interval's coverage $\hat C$ (and its $95\%$ confidence interval), mean CrI length, and MSE of the posterior mean for the generalized posteriors of each strategy $S\in\{\mathrm{RA,IPW,AIPW}\}$ on the IHDP-100 ($R=100$) and IHDP-1000 ($R=1000$) replications under a diffuse prior $\mathcal{N}(0,100^2)$. On IHDP, the $\mathrm{AIPW}$ posterior under-covers.
}
}
\label{table:e8_ihdp_ate}
\end{table}

\newpage

\FloatBarrier
\section{Experiment Details}
\label{app:implementation}

\subsection{Implementation Details}

The code for all experiments is available at \url{https://github.com/EmilJavurek/Generalised-Bayes-for-Causal-Inference}. Here, we provide a short description of the essential implementation details.

\textbf{Nuisances:}
For the nine-setting ATE benchmark and the baseline configurations of our sensitivity experiments, we fit logistic regression for the propensity and two ridge regressions for the outcomes (T-learner). These models use a shared, fixed feature map
$\Phi(X) = [X, X^2, \sin(X), x_{\mathrm{extra}}, 1]$, where $x_{\mathrm{extra}}=X_1X_2$ if $d\ge 2$ (or $X_1^3$ if $d=1$).
Propensities are regularized with $\lambda_{\text{prop}}$ and clipped to $[\epsilon,1-\epsilon]$ to stabilize IPW/AIPW. We use cross-fitting: nuisance models are fit on each training fold and pseudo-outcomes are evaluated on held-out folds. The nuisance-estimation sensitivity experiments additionally consider tree ensembles, neural networks, and Bayesian Gaussian process posterior-mean plug-ins.

\textbf{Prior:}
For the nine-setting ATE benchmark, we use a Gaussian prior $\mathcal{N}(\mu_0,\sigma_0^2)$ with $\mu_0=0$ and $\sigma_0=1$. The additional ATE robustness and sensitivity experiments instead use a diffuse $\mathcal N(0,100^2)$ prior, unless otherwise specified. In the prior-strength experiment, we explicitly vary both the prior mean and variance. For CATE, the prior is a Gaussian process with the kernel specified in the configuration; we mainly use RBF or Matern kernels.

\textbf{Variational family:}
For the nine-setting ATE benchmark, the variational family is a univariate Gaussian with parameters $(\mu,\log\sigma)$; optimizing in $\log\sigma$ enforces positivity of $\sigma$ and yields stable optimization. When optimization is stopped early, the learned posterior variance can be slightly conservative relative to the analytic closed form.
For CATE, the variational family is a sparse inducing-point GP with $M$ inducing locations initialized from the training covariates. In the example CATE configuration, we use $M=20$ and a Matern-$5/2$ kernel with lengthscale 2.0, variance 2.0, and jitter $10^{-4}$. A constant mean function is fit to the pseudo-outcomes and subtracted before GP fitting, and then added back to predictions.

\textbf{Variational inference (VI):}
We use the Adam optimizer. For the nine-setting ATE benchmark, we use a batch size of 200 (number of $\theta$ samples per update), learning rate 0.03, and 2000 epochs. For CATE, the batch size field is ignored because the GP loss is evaluated in full batch.

\textbf{Additional ATE experiments:}
For the misspecification, calibration, prior-strength, cross-fitting, nuisance-estimation, and semi-synthetic evaluations, we compute the generalized posterior analytically using Gaussian conjugacy for squared-error pseudo-outcome losses, rather than VI. Unless otherwise specified, we use $K=5$ cross-fitting folds and the analytic plug-in calibration $\hat\omega=\widehat{\Var}(\hat Y^S)^{-1}$. The calibration comparison also considers a bootstrap estimate obtained from $100$ resamples of the fixed cross-fitted pseudo-outcomes, without refitting nuisance models. This differs from the coverage-matching bootstrap calibration discussed in Section~\ref{sec:omega-calibration-explanation}. In the cross-fitting and nuisance-learner ablations, we hold $\hat\omega$ fixed at the value obtained from the $K=5$ linear/ridge baseline on the same dataset, isolating the effects of the respective choices. The Bayesian nuisance experiment uses fold-wise posterior-mean plug-ins without integrating over nuisance posterior uncertainty.

\subsection{Data-Generating Processes}
\subsubsection{Back-door adjustment setting}
\label{app:dgp-backdoor-description}

\begin{enumerate}

\item \emph{$\mathcal{D}_1$ (Linear, confounded, homoskedastic).}
\begin{align}
X &\sim \mathcal{N}(0, I_2), \nonumber\\
e(X) &= \sigma(X^\top \beta), \qquad \sigma(u)=\frac{1}{1+e^{-u}}, \nonumber\\
A \mid X &\sim \mathrm{Bernoulli}\!\bigl(e(X)\bigr), \nonumber\\
Y \mid (A,X) &= \tau\,A + X^\top \gamma + \varepsilon, \qquad \varepsilon\sim \mathcal{N}(0,1). \nonumber
\end{align}
The CATE is $\tau(X)=\tau$ and the ATE is $\mathrm{ATE}=\tau$.

\item \emph{$\mathcal{D}_2$ (Linear CATE heterogeneity, mean-zero covariates).}
\begin{align}
X &\sim \mathcal{N}(0, I_2), \nonumber\\
e(X) &= \sigma(X^\top \beta), \nonumber\\
A \mid X &\sim \mathrm{Bernoulli}\!\bigl(e(X)\bigr), \nonumber\\
Y \mid (A,X) &= \bigl(\theta_0 + X^\top \theta\bigr)A + X^\top \gamma + \varepsilon, \qquad \varepsilon\sim \mathcal{N}(0,1). \nonumber
\end{align}
The CATE is $\tau(X)=\theta_0 + X^\top\theta$ and the ATE is $\mathrm{ATE}=\theta_0$.

\item \emph{$\mathcal{D}_3$ (Linear CATE heterogeneity, nonzero mean covariates).}
\begin{align}
X &\sim \mathcal{N}(\mu, I_2), \nonumber\\
e(X) &= \sigma(X^\top \beta), \nonumber\\
A \mid X &\sim \mathrm{Bernoulli}\!\bigl(e(X)\bigr), \nonumber\\
Y \mid (A,X) &= \bigl(\theta_0 + X^\top \theta\bigr)A + X^\top \gamma + \varepsilon, \qquad \varepsilon\sim \mathcal{N}(0,1). \nonumber
\end{align}
The CATE is $\tau(X)=\theta_0 + X^\top\theta$ and the ATE is $\mathrm{ATE}=\theta_0 + \theta^\top\mu$.

\item \emph{$\mathcal{D}_4$ (Nonlinear outcome; quadratic and interaction CATE).}
\begin{align}
X &\sim \mathcal{N}(0, I_2), \nonumber\\
e(X) &= \sigma(X^\top \beta), \nonumber\\
A \mid X &\sim \mathrm{Bernoulli}\!\bigl(e(X)\bigr), \nonumber\\
f(X) &= X_1^2 + \sin(X_2), \nonumber\\
Y \mid (A,X) &= \bigl(\alpha_0 + \alpha_1 X_1X_2\bigr)A + f(X) + \varepsilon, \qquad \varepsilon\sim \mathcal{N}(0,1). \nonumber
\end{align}
The CATE is $\tau(X)=\alpha_0 + \alpha_1 X_1X_2$ and the ATE is $\mathrm{ATE}=\alpha_0$.

\item \emph{$\mathcal{D}_5$ (Nonlinear propensity; linear outcome).}
\begin{align}
X &\sim \mathcal{N}(0, I_2), \nonumber\\
z(X) &= b_0 + b_1 X_1 + b_2 X_1^2 + b_3 \sin(X_2), \nonumber\\
e(X) &= \sigma\!\bigl(z(X)\bigr), \nonumber\\
A \mid X &\sim \mathrm{Bernoulli}\!\bigl(e(X)\bigr), \nonumber\\
h(X) &= X_1 + \tfrac12 X_1^2 + \tfrac12 \sin(X_2), \nonumber\\
Y \mid (A,X) &= \tau\,A + h(X) + \varepsilon, \qquad \varepsilon\sim \mathcal{N}(0,1). \nonumber
\end{align}
The CATE is $\tau(X)=\tau$ and the ATE is $\mathrm{ATE}=\tau$.

\item \emph{$\mathcal{D}_6$ (Limited overlap).}
\begin{align}
X &\sim \mathcal{N}(0, I_2), \nonumber\\
e(X) &= \sigma(3.5 + 3X_1), \nonumber\\
A \mid X &\sim \mathrm{Bernoulli}\!\bigl(e(X)\bigr), \nonumber\\
Y \mid (A,X) &= \tau\,A + X^\top \gamma + \varepsilon, \qquad \varepsilon\sim \mathcal{N}(0,1). \nonumber
\end{align}
The CATE is $\tau(X)=\tau$ and the ATE is $\mathrm{ATE}=\tau$.

\item \emph{$\mathcal{D}_7$ (Heteroskedastic heavy-tailed noise).}
\begin{align}
X &\sim \mathcal{N}(0, I_2), \nonumber\\
e(X) &= \sigma(X^\top \beta), \nonumber\\
A \mid X &\sim \mathrm{Bernoulli}\!\bigl(e(X)\bigr), \nonumber\\
\sigma(X) &= \exp\!\bigl(0.5 X_1\bigr), \nonumber\\
\eta &\sim t_{\nu}, \nonumber\\
Y \mid (A,X) &= \tau\,A + X^\top \gamma + \sigma(X)\,\eta. \nonumber
\end{align}
The CATE is $\tau(X)=\tau$ and the ATE is $\mathrm{ATE}=\tau$.

\item \emph{$\mathcal{D}_8$ (High-dimensional sparse confounding).}
\begin{align}
X &\sim \mathcal{N}(0, I_p), \nonumber\\
\beta_j &= 0,\ \gamma_j=0 \quad \text{for } j>s, \nonumber\\
\beta_{1:s} &= \mathrm{linspace}(0.2,1.0,s), \qquad \gamma_{1:s}=\mathrm{linspace}(1.0,0.2,s), \nonumber\\
e(X) &= \sigma(X^\top \beta), \nonumber\\
A \mid X &\sim \mathrm{Bernoulli}\!\bigl(e(X)\bigr), \nonumber\\
Y \mid (A,X) &= \tau\,A + X^\top \gamma + \varepsilon, \qquad \varepsilon\sim \mathcal{N}(0,1). \nonumber
\end{align}
The CATE is $\tau(X)=\tau$ and the ATE is $\mathrm{ATE}=\tau$.

\item \emph{$\mathcal{D}_9$ (Friedman-style nonlinear CATE).}
\begin{align}
X &\sim \mathrm{Unif}\!\bigl([0,1]^5\bigr), \nonumber\\
e(X) &= \sigma\!\left(-0.5 + X_1 - 0.25X_2 + 0.25X_3\right), \nonumber\\
A \mid X &\sim \mathrm{Bernoulli}\!\bigl(e(X)\bigr), \nonumber\\
\mu(X) &= 10\sin(\pi X_1X_2)+20(X_3-0.5)^2+10X_4+5X_5, \nonumber\\
\tau(X) &= 1+\frac{X_1}{X_2+0.1}, \nonumber\\
Y \mid (A,X) &= \mu(X) + \tau(X)A + \varepsilon, \qquad \varepsilon\sim \mathcal{N}(0,1). \nonumber
\end{align}
The CATE is $\tau(X)=1+\frac{X_1}{X_2+0.1}$ and the ATE is $\mathrm{ATE}=1+\tfrac12\log(11)$.

\end{enumerate}
\section{Limitations}
\label{appx:limitations}

Our framework inherits the assumptions required for causal identification and for orthogonal statistical learning. In the back-door ATE and CATE settings considered here, violations of consistency, positivity, unconfoundedness, no interference, or accurate measurement can lead the generalized posterior to quantify uncertainty around an identified statistical functional that differs from the causal estimand of interest. The theoretical guarantees further rely on regularity conditions such as Bernstein--von Mises approximations, Neyman-orthogonality, cross-fitting, and sufficient nuisance-estimation rates; these may fail under weak identification, unstable nuisance learning, highly irregular losses, small samples, or very high-dimensional targets. Practical performance can also depend on implementation choices such as nuisance learners, variational approximations, Gaussian process kernels, inducing points, and optimizer settings.

\section*{Broader Impact}
\label{appx:broader-impact}

This work develops a methodological framework for uncertainty quantification in causal inference. A potential positive impact is that calibrated uncertainty estimates can help practitioners avoid overconfident causal conclusions, especially in domains where treatment-effect estimates may inform consequential decisions such as medicine, public policy, education, or economics. By placing priors directly on causal estimands and using identification-driven losses, the framework may also make uncertainty statements more transparent and easier to scrutinize than approaches that encode beliefs indirectly through high-dimensional nuisance models.

At the same time, causal uncertainty estimates are only meaningful under the assumptions required for identification and estimation. If the relevant causal assumptions fail, for example due to unmeasured confounding, violations of positivity, interference, measurement error, or distribution shift, then calibrated-looking credible intervals may still support incorrect causal conclusions. This risk is especially important in high-stakes applications, where erroneous treatment-effect estimates could lead to harmful or unfair decisions. Moreover, because the proposed framework can be combined with flexible machine learning nuisance estimators, it may inherit biases present in the data or in the chosen learning pipeline, including biases that affect subpopulations differently.

We therefore view the proposed framework as a tool for uncertainty-aware causal analysis rather than as a substitute for domain expertise, study design, or sensitivity analysis. In practical deployments, users should explicitly state the causal identification assumptions, assess overlap and nuisance-model quality, report calibration and sensitivity diagnostics, and avoid using the resulting posterior summaries as standalone justification for high-stakes interventions. The experiments in this paper comprise synthetic simulations and a semi-synthetic IHDP benchmark; we do not release human-subject datasets, pretrained generative models, or other high-risk assets.



\newpage
\section*{NeurIPS Paper Checklist}

\begin{enumerate}

\item {\bf Claims}
    \item[] Question: Do the main claims made in the abstract and introduction accurately reflect the paper's contributions and scope?
    \item[] Answer: \answerYes{} 
    \item[] Justification: The abstract and introduction state the proposed generalized Bayesian framework, its scope across ATE/CATE and loss-based causal estimators, the Neyman-orthogonal robustness guarantee, and the empirical claim of calibrated uncertainty. These claims are developed in Sections~5--6 and Appendix~A.

\item {\bf Limitations}
    \item[] Question: Does the paper discuss the limitations of the work performed by the authors?
    \item[] Answer: \answerYes{} 
    \item[] Justification: Appendix~\ref{appx:limitations} discusses limitations arising from causal identification assumptions, regularity conditions for the theoretical guarantees, nuisance estimation, calibration, and practical implementation choices.

\item {\bf Theory assumptions and proofs}
    \item[] Question: For each theoretical result, does the paper provide the full set of assumptions and a complete (and correct) proof?
    \item[] Answer: \answerYes{} 
    \item[] Justification: The paper states the assumptions for the main posterior-stability result and provides formal finite-dimensional and projected infinite-dimensional versions with proofs in Appendix~A. The informal theorem in Section~5.4 is supported by the formal statements, assumptions, and proof details in the appendix.

    \item {\bf Experimental result reproducibility}
    \item[] Question: Does the paper fully disclose all the information needed to reproduce the main experimental results of the paper to the extent that it affects the main claims and/or conclusions of the paper (regardless of whether the code and data are provided or not)?
    \item[] Answer: \answerYes{} 
    \item[] Justification: Section~\ref{sec:experiments} and Appendices~\ref{app:other-experiments} and~\ref{app:implementation} describe the experimental protocols, evaluation metrics, synthetic data-generating processes, semi-synthetic IHDP benchmark, nuisance estimation, calibration, and posterior computation.

\item {\bf Open access to data and code}
    \item[] Question: Does the paper provide open access to the data and code, with sufficient instructions to faithfully reproduce the main experimental results, as described in supplemental material?
    \item[] Answer: \answerYes{} 
    \item[] Justification: The paper provides a link to the experiment code repository. Appendix~\ref{app:implementation} documents the implementation and synthetic data-generating processes, while Appendix~\ref{app:other-experiments} reports additional experiments, including the semi-synthetic IHDP evaluation.

\item {\bf Experimental setting/details}
    \item[] Question: Does the paper specify all the training and test details (e.g., data splits, hyperparameters, how they were chosen, type of optimizer) necessary to understand the results?
    \item[] Answer: \answerYes{} 
    \item[] Justification: Section~\ref{sec:experiments} and Appendix~\ref{app:implementation} specify the experimental settings, sample sizes, repetitions, nuisance estimators, cross-fitting, priors, posterior computation, optimization settings, and calibration procedures. Additional experimental configurations are described in Appendix~\ref{app:other-experiments}.

\item {\bf Experiment statistical significance}
    \item[] Question: Does the paper report error bars suitably and correctly defined or other appropriate information about the statistical significance of the experiments?
    \item[] Answer: \answerYes{} 
    \item[] Justification: The paper reports empirical coverage with 95\% confidence intervals and reports credible-interval lengths with standard deviations across repetitions. The variability is computed over repeated runs with different random seeds/synthetic samples.

\item {\bf Experiments compute resources}
    \item[] Question: For each experiment, does the paper provide sufficient information on the computer resources (type of compute workers, memory, time of execution) needed to reproduce the experiments?
    \item[] Answer: \answerYes{} 
    \item[] Justification: The experiments have been run on a laptop \texttt{CPU} with trivial runtime (<1hr).
    
\item {\bf Code of ethics}
    \item[] Question: Does the research conducted in the paper conform, in every respect, with the NeurIPS Code of Ethics \url{https://neurips.cc/public/EthicsGuidelines}?
    \item[] Answer: \answerYes{} 
    \item[] Justification: The work is methodological and is evaluated using synthetic simulations and a semi-synthetic benchmark based on existing data. No new human-subject data are collected, and no deployed decision systems are involved. To the best of the authors' knowledge, the research conforms to the NeurIPS Code of Ethics.

\item {\bf Broader impacts}
    \item[] Question: Does the paper discuss both potential positive societal impacts and negative societal impacts of the work performed?
    \item[] Answer: \answerYes{} 
    \item[] Justification: The Broader Impact section discusses potential benefits of calibrated causal uncertainty quantification and risks associated with identification failures, model misspecification, biased nuisance estimation, and the misuse of causal conclusions in consequential decisions.
    
\item {\bf Safeguards}
    \item[] Question: Does the paper describe safeguards that have been put in place for responsible release of data or models that have a high risk for misuse (e.g., pre-trained language models, image generators, or scraped datasets)?
    \item[] Answer: \answerNA{} 
    \item[] Justification: The paper does not release high-risk models, scraped datasets, or other assets with a high risk of misuse. The released assets consist of code for methodological evaluation using synthetic and semi-synthetic experiments.

\item {\bf Licenses for existing assets}
    \item[] Question: Are the creators or original owners of assets (e.g., code, data, models), used in the paper, properly credited and are the license and terms of use explicitly mentioned and properly respected?
    \item[] Answer: \answerNA{} 
    \item[] Justification: The experiments use synthetic data-generating processes rather than existing datasets, pretrained models, or third-party data assets. Standard software dependencies, if any, should be documented in the released code repository.

\item {\bf New assets}
    \item[] Question: Are new assets introduced in the paper well documented and is the documentation provided alongside the assets?
    \item[] Answer: \answerYes{} 
    \item[] Justification: The paper releases code for the proposed generalized Bayes causal inference experiments. The accompanying appendix documents the experimental setup, synthetic data-generating processes, nuisance estimators, priors, and optimization details.

\item {\bf Crowdsourcing and research with human subjects}
    \item[] Question: For crowdsourcing experiments and research with human subjects, does the paper include the full text of instructions given to participants and screenshots, if applicable, as well as details about compensation (if any)? 
    \item[] Answer: \answerNA{} 
    \item[] Justification: The research does not recruit participants, conduct crowdsourcing, or collect new human-subject data. The semi-synthetic IHDP evaluation uses an existing benchmark with simulated outcomes.

\item {\bf Institutional review board (IRB) approvals or equivalent for research with human subjects}
    \item[] Question: Does the paper describe potential risks incurred by study participants, whether such risks were disclosed to the subjects, and whether Institutional Review Board (IRB) approvals (or an equivalent approval/review based on the requirements of your country or institution) were obtained?
    \item[] Answer: \answerNA{} 
    \item[] Justification: The work involves no participant recruitment or new human-subject data collection. The experiments use synthetic data and an existing semi-synthetic benchmark; no new human-subject study was conducted.

\item {\bf Declaration of LLM usage}
    \item[] Question: Does the paper describe the usage of LLMs if it is an important, original, or non-standard component of the core methods in this research? Note that if the LLM is used only for writing, editing, or formatting purposes and does \emph{not} impact the core methodology, scientific rigor, or originality of the research, declaration is not required.
    \item[] Answer: \answerNA{} 
    \item[] Justification: The core method development, theoretical results, and experiments do not involve LLMs as an important, original, or non-standard component. Any LLM use solely for writing, editing, or formatting would not require declaration under the checklist question.

\end{enumerate}

\end{document}